\documentclass[11pt,reqno]{amsart}
\usepackage{amsmath,amssymb,amsfonts,amsthm,enumerate,natbib,color,ifthen,hyperref}
\usepackage{xargs}
\usepackage[textwidth=4cm, textsize=footnotesize]{todonotes}
\usepackage[latin1]{inputenc}

\usepackage{float}
\usepackage{caption}
\usepackage{subcaption}
\usepackage{graphicx}

\usepackage{aliascnt,bbm}
\def\rset{\mathbb R}
\def\zset{\mathbb Z}

\def\eqsp{\;}

\newcommand{\pscal}[2]{\left\langle#1,#2\right\rangle}

\newcommand{\eqdef}{\ensuremath{\stackrel{\mathrm{def}}{=}}}

\def\Zset{\mathsf{Z}} % Espace d 'etat
\def\Yset{\mathsf{Y}} % Espace d 'etat
\def\F{\mathcal{F}} % filtration
\def\e{\mathcal{E}}

\def\N{\mathcal{N}}

\def\D{\mathcal{D}}
\def\R{\mathcal{R}}
\def\A{\mathcal{A}}

\newcommandx\sequence[3][2=t,3=\zset]
{\ifthenelse{\equal{#3}{}}{\ensuremath{\{ #1_{#2}\}}}{\ensuremath{\{ #1_{#2}, \eqsp #2 \in #3 \}}}}

\def\PP{\mathbb{P}} % proba
\newcommand{\CPP}[3][]
{\ifthenelse{\equal{#1}{}}{{\mathbb P}\left(\left. #2 \, \right| #3 \right)}{{\mathbb P}_{#1}\left(\left. #2 \, \right | #3 \right)}}
\def\PE{\mathbb{E}} % esperance
\newcommand{\CPE}[3][]
{\ifthenelse{\equal{#1}{}}{{\mathbb E}\left[\left. #2 \, \right| #3 \right]}{{\mathbb E}_{#1}\left[\left. #2 \, \right | #3 \right]}}

\def\wtilde{\widetilde} % proba

\def\W{\mathcal{W}}

\def\Cset{\mathcal{C}} % Petite set
\def\Prox{\operatorname{Prox}}

\def\r{\textsf{r}}

\usepackage{bm}%bold matH

\theoremstyle{plain}
\newtheorem{theorem}{Theorem}
\newtheorem{assumption}{H\hspace{-3pt}}

\newaliascnt{proposition}{theorem}

\aliascntresetthe{proposition}

\newaliascnt{lemma}{theorem}
\newtheorem{lemma}[lemma]{Lemma}
\aliascntresetthe{lemma}
\newaliascnt{corollary}{theorem}
\newtheorem{corollary}[corollary]{Corollary}
\aliascntresetthe{corollary}

\theoremstyle{definition}
\newaliascnt{definition}{theorem}

\aliascntresetthe{definition}

\newaliascnt{remark}{theorem}
\newtheorem{remark}[remark]{Remark}
\aliascntresetthe{remark}

\newaliascnt{example}{theorem}

\aliascntresetthe{example}

\def\rmd{\mathrm{d}}

\def\1{\mathbbm{1}}

\DeclareMathOperator*{\argmin}{{\textsf{Argmin}}}

\begin{document}

\title[Algorithm unrolling models for inverse problems]{A statistical perspective on algorithm unrolling models for inverse problems}
%\thanks{This work is partially supported by the NSF grant DMS 1513040}

\author{Yves Atchad\'e}\thanks{ Y. Atchad\'e: Boston University, 111 Cummington Mall, Boston 02215 MA, United States. {\em E-mail address:} atchade@bu.edu}
\author{Xinru Liu}\thanks{ X. Liu: Boston University, 111 Cummington Mall, Boston 02215 MA, United States. {\em E-mail address:} xinruliu@bu.edu}
\author{Qiuyun Zhu}\thanks{ Q. Zhu: University of Minnesota, 224 Church St SE, Minneapolis 55455 MN, United States. {\em E-mail address:} qzhu@umn.edu}
  


\maketitle

\begin{center} (April 2023) \end{center}

\begin{abstract}
We consider inverse problems where the conditional distribution of the observation ${\bf y}$ given the latent variable of interest ${\bf x}$ (also known as the forward model) is known, and we have access to a data set in which multiple instances of ${\bf x}$ and  ${\bf y}$ are both observed. In this context, algorithm unrolling has become a very popular approach for designing state-of-the-art deep neural network architectures that effectively exploit the forward model. We analyze the statistical complexity of the gradient descent network (GDN), an algorithm unrolling architecture driven by proximal gradient descent. We show that the unrolling depth needed for the optimal statistical performance of GDNs is of order $\log(n)/\log(\varrho_n^{-1})$, where $n$ is the sample size, and $\varrho_n$ is the convergence rate of the corresponding gradient descent algorithm. We also show that when the negative log-density of the latent variable ${\bf x}$ has a simple proximal operator, then a GDN unrolled at depth $D'$ can solve the inverse problem at the parametric rate $O(D'/\sqrt{n})$. Our results thus also suggest that algorithm unrolling models are prone to overfitting as the unrolling depth $D'$ increases. We provide several examples to illustrate these results.
\end{abstract}

\section{Introduction}
\label{sec:intro}
Inverse problems are common problems in science and engineering where one seeks information on a latent variable of interest, given some related observation.  We consider an inverse problem with a  latent quantity of interest ${\bf x}\in\rset^{d_x}$ that is related to the observed variable ${\bf y}\in\rset^{d_y}$ through the so-called forward statistical model
\begin{equation}\label{f:model}
{\bf y} \;\vert\; {\bf x} \; \sim  e^{- f({\bf y} \vert {\bf x})}\rmd {\bf y},\end{equation} 
for some function $f(\cdot\vert {\bf x}):\;\rset^{d_y}\to\rset$. Throughout the paper, unless otherwise stated, all model densities are defined with respect to the corresponding Lebesgue measure. Although the function $f$ is unknown in general, we focus in this work on inverse problems for which the forward model is well-understood and $f$ is known. This is the case with many inverse problems in imaging. %However the methodology developed below can be easily extended to deal with cases where $f$ is unknown and replaced by a parametric model.  
An important special case in the applications is the Gaussian linear  model corresponding (up to an additive constant that we ignore) to 
\begin{equation}\label{f:model:lin}
f({\bf y} \vert {\bf x}) = \frac{1}{2v^2} \|{\bf y} - A{\bf x}\|_2^2,\end{equation}
with known parameters $v>0$, and $A\in\rset^{d_y\times d_x}$. When the inverse problem is ill-posed, additional knowledge is fundamental for good recovery of ${\bf x}$. For example in the linear regression model  (\ref{f:model:lin}), it is well-known that without any additional assumption, the minimax optimal rate in the estimation of ${\bf x}$ is of order $\sqrt{d_x/d_y}$. However this rate can be improved  if ${\bf x}$ is known to possess some additional features such as smoothness or sparsity.  A Bayesian perspective  is particularly simple. If $\mu_0$ denotes a prior distribution that encodes the information available on ${\bf x}$, then ${\bf x}$ is inferred using its posterior distribution 
\begin{equation}\label{class:BIP}
\pi_{\mu_0}\left(\rmd {\bf x}\vert {\bf y}\right) \propto \mu_0(\rmd {\bf x}) e^{-f({\bf y} \vert {\bf x})}.\end{equation}
%or the related penalized maximum likelihood estimate
%\begin{equation}\label{class:IP}
%\hat{\bf x} =\argmin_{{\bf u}\in\Xset} \left[ f({\bf y} \vert {\bf u}) + \mathcal{R}({\bf u})\right],\end{equation}
%for some regularization function $\mathcal{R}:\;\Xset\to\rset$.  
Inverse problems have a long history in statistics and applied mathematics, and the posterior distribution  in (\ref{class:BIP}) as well as related penalized estimators are the backbone of rigorous inference \cite{bissantz:etal:2007,stuart_2010,knapik:etal:2011,blanchard:etal:18,rastogi:etal:2020}.   When valid information are available on ${\bf x}$ and appropriately encoded in $\mu_0$, the posterior distribution $\pi_{\mu_0}$ can enjoy better statistical properties than say, the minimizer of ${\bf x} \mapsto f({\bf y}\vert {\bf x})$. However, finding such good prior distributions is often very challenging in many applications. 
%Another limitation of this classical approach is that sampling from (\ref{class:BIP})  can be computationally challenging, for instance in problems where a large number of inversions need to be solved (as for instance in remote sensing applications).

\subsection{Learning to solve inverse problems}
In a growing number of settings, particularly in image restoration tasks, researchers have access to datasets in which the  latent variable ${\bf x}$ and the related observation ${\bf y}$ are both observed. Indeed such datasets can often be simulated in settings where  $f$ is known. Hence, suppose that we have a dataset $\D = \{({\bf x}_i, {\bf y}_i),\;1\leq i\leq n\}$ of i.i.d. samples, such that for $1\leq i \leq n$,
\begin{equation}\label{data:model}
 {\bf x}_i\sim \mu,\;\;\;\mbox{ and }\;\;\; {\bf y}_i  \; \vert \; {\bf x}_i \sim e^{- f({\bf y} \vert {\bf x}_i)}\rmd {\bf y},
 \mbox{ and where }\; \mu(\rmd {\bf x}) = \frac{1}{c_\mu}e^{-\mathcal{R}({\bf x})} \rmd{\bf x},
\end{equation}
for some function $\mathcal{R}:\rset^{d_x}\to\rset$, and a normalizing constant $c_\mu$. Hence under (\ref{data:model}), $\mu$ is the marginal distribution of the latent variables. The conditional distribution of ${\bf x}_i$ given ${\bf y}_i$ is then given by
\[\pi(\rmd {\bf x}\vert {\bf y}_i)\propto \exp\left(-\mathcal{R}({\bf x}) -f({\bf y}_i \vert {\bf x})\right)\rmd {\bf x},\]
and its modal value is given by the function $g:\;\rset^{d_y}\to \rset^{d_x}$ with
\begin{equation}\label{def:g}
g({\bf y}) \eqdef \argmin_{{\bf x}\in\rset^{d_x}} \left[ f({\bf y} \vert {\bf x}) + \mathcal{R}({\bf x})\right].
\end{equation}
We will assume below that $g({\bf y})$ is uniquely defined. We stress again that the distribution $\mu$ in (\ref{data:model})  is not a prior distribution of ${\bf x}$ as selected by the researcher, but the actual marginal distribution  of ${\bf x}$ unknown  to the researcher.  Hence $g$ and $\pi(\cdot\vert{\bf y})$ are typically unknown. In fact, one of the key challenges in inverse problems is building a prior distribution $\mu_0$  that is as close as possible to $\mu$ so that the resulting posterior distribution as given in (\ref{class:BIP}) approximates well the corresponding conditional distribution. 

In keeping with the assumption that ${\bf x}$ possesses additional structures, in many inverse problems the support of the marginal distribution $\mu$ lays in a much smaller (but unknown) subspace of $\rset^{d_x}$.  As a result of such marginal distribution concentration, it is often the case that the conditional distribution of ${\bf x}$ given ${\bf y}$ is also tightly  concentrated around $g({\bf y})$, in the sense that
\begin{equation}\label{repres:intro}
{\bf x}_i = g({\bf y}_i) + \bm{\xi}_i,\mbox{ where }\PE(\bm{\xi}_i\;\vert \; {\bf y}_i) \approx 0,\; 1\leq i\leq n.\end{equation}
The representation (\ref{repres:intro}) makes clear that in such settings where we have  an informative (but unknown) marginal distribution, and given a  dataset $\D$, one can learn the function $g$ by regressing ${\bf x}$ on ${\bf y}$.  In other words, we can learn to solve directly the inverse problem by regression using the dataset $\D$. The approach has become  popular in computational imaging (\cite{berger:etal:2012,xie:xu:2012,8253590,yang:sun:li:xu:16,ravishankar:chun:fessler:17,aggarwal:mani:etal:17,chun:fessler:18,zhang:2017,liu:cheng:ma:long:xin:luo:19, li:eldar:etal:20}). A remarkable contribution of this literature is a number of specific  deep neural network  architectures generally called algorithm unrolling networks that leverage  the structure of the forward model (\cite{gregor:2010:learning, sreter2018learned,sulam:etal:20,tolooshams2020convolutional}), see also the reviews (\cite{ongie:etal:20, shlezinger:etal:21, monga2021algorithm}). 

However a fundamental question that has not been addressed in the literature so far is how well one can estimate the function $g$ using these unrolling-based deep neural network architectures.

\subsection{Main contributions}
To address this problem, and assuming that the data generating process (\ref{repres:intro}) holds, we consider  the nonparametric  regression model 
\begin{equation}\label{au:model}
{\bf x}_i = g_W({\bf y}_i) + \bm{\epsilon}_i,\;\;\; i=1,\ldots,n,\end{equation}  
 with regression errors $\bm{\epsilon}_i\stackrel{i.i.d.}{\sim}\textbf{N}(0,\sigma^2I_{d_x})$, for some positive variance parameter $\sigma^2$  taken as known for simplicity, and for a function class  $\{g_W,\;W\in\W\}$, where  $g_W:\;\rset^{d_y}\to\rset^{d_x}$ is a gradient descent neural  network (GDN) function obtained by unrolling $D'$ times a parametrized proximal gradient descent algorithm for solving (\ref{def:g}). We give precise definition below.  The architecture thus makes explicit use of the forward map $f$. We develop a sparse Bayesian framework for estimating (\ref{au:model}) using a spike-and-slab prior distribution on the parameter $W$, and we analyze the statistical performance of the resulting estimator of $g$. We focus on the setting where the function $\R$ is convex, but not necessarily differentiable. Under some additional regularity conditions, and ignoring logarithmic terms, we show that GDN  for estimating $g$ achieves the statistical error rate 
 \[C_1(D')\times n^{-\frac{1}{2 + d_x}},\]
for some constant $C_1(D')$ that depends on the unrolling depth $D'$, but also on the input dimensions $d_x,d_y$. Keeping dimensions and the number of unrolling $D'$ fixed, the result implies for example that when $d_y\geq d_x$, the GDN architecture, by making explicit use of the forward model, achieves a better rate than the minimax rate $C_2n^{-\frac{1}{2 + d_y}}$ for estimating $g:\;\rset^{d_y}\to\rset^{d_x}$ viewed as a Lipschitz function. We note however that these constants $C_1(D'),C_2$ can depend poorly\footnote{The poor dependence on the input dimensions is not specific to our work, and is rooted in the current state of knowledge in deep learning approximation theory (see e.g. \cite{yarostky:17})} on the dimensions $d_x,d_y$.

The convergence rate of the estimator can be faster than the aforementioned rate. Indeed, we also show that when the proximal map of $\R$ is simple and can be well-approximated by a simple neural network function, then  the GDN architecture achieves a faster statistical rate. For instance, if $\mu$ is as in (\ref{ex:mu}) below, a common assumption in image restoration tasks, then ignoring log terms, our result shows that the GDN achieves the parametric rate $C_3\times D' / \sqrt{n}$, for some dimension-dependent constant $C_3$. Importantly, our result thus suggests that the statistical performance of a GDN unrolled at depth $D'$ deteriorates as $D'$ increases, implying an overfitting phenomenon. Although we do not have a matching lower bound theory to confirm this overfitting phenomenon, we have performed  extensive numerical experiments that  all show an overfitting of the model as $D'$ increases.

One of the practical challenges in building a GDN is the lack of theoretical guidelines in the choice of the depth $D'$. An offshoot of our theoretical analysis is the derivation that the best performance of a GDN is achieved by scaling the network depth as $D'\sim \log(n)/\log(\varrho_n^{-1})$, where $\varrho_n$ is the convergence rate of the proximal gradient algorithm for solving (\ref{def:g}).

 \subsection{Related work}
Most of the existing theoretical results on algorithm unrolling have studied the approximation capability of the resulting function class in the linear case. For instance \cite{chen:liu:wang:yin:18} studied the capability of the GDN function class to recover directly the signal ${\bf x}$ in the linear model (\ref{f:model:lin}).   \cite{gilton:etal:20} proposed a novel unrolling architecture based on the Neumann series identity, and studied its approximation capability in the noiseless version of the linear model (\ref{f:model:lin}). To the best of our knowledge, our work is the first to analyze the statistical properties of algorithm unrolling in a way that accounts for both its approximation capability \textit{and} its complexity.

Several prior works have also considered the statistical complexity of other deep learning models using a similar nonparametric regression setting where regularization is explicitly introduced to control model complexity \cite{barron:18,shieber:20,taheri:etal:21,e:etal:20}. Our framework is closer to \cite{polson:18} and employs a Bayesian approach. However none of these results can be directly applied to algorithm unrolling architectures. Another unique feature of our framework that is worth emphasizing is that it produces  posterior distributions that are computationally tractable using the sparse asynchronous SGLD of (\cite{atchade:wang:2021fast}).

%\footnote{Deep learning models are often deployed in the applications without any explicit regularization technique, and it is widely believed that stochastic gradient descent and its variants have implicit regularization mechanisms that make these models immune to overfitting. However these issues are not completely well-understood.}

Finally we contrast  our nonparametric  regression approach with the two-step approach proposed for instance by \cite{chang:etal:onsta:17}, where the proximity operator of $\R$ is first estimated from the dataset $\D$,  and $g$ is then estimated by solving  (\ref{def:g}) using the estimated proximal operator obtained from the first step. The strategy seems statistically sub-optimal because the estimation of the proximal operator requires estimating the density of $\mu$ in general, which is a fundamentally more difficult statistical problem  (\cite{samworth:18}). However it is conceivable that an adaptive density estimation method may exist that achieve better rates  for densities with simple  proximal maps, thus matching the approach developed here. More research is needed on this topic.

\subsection{Outline of the paper}
The remainder of the paper is organized as follows. We close this introduction with some general notations. The main results are described in Section \ref{sec:main:desc}.  
The results are obtained using a more general Bayesian posterior contraction result of independent interest that we described in Section \ref{sec:gen:post:contraction}. Some supporting numerical illustrations are presented in Section \ref{sec:illust}. All the proofs are postponed to Appendix \ref{sec:proofs}.

\subsection{Notations}
We define the sub-Gaussian norm of a probability measure $\nu$ on $\rset^d$ with expected value $m$  as the smallest constant $c$ for which the following holds 
\[\int_{\rset^d} e^{\pscal{u}{z-m}}\nu(\rmd z) \leq e^{\frac{c^2\|u\|_2^2}{2}},\;\;\mbox{ for all } u\in\rset^d.\]
%This definition differs (up to a constant) from the classical definition based on the Orlicz norm (see e.g. \cite{vershynin:18}). We prefer this definition as it allows for easy tracking of constants. 
If $Z$ is a random variable with distribution $\nu$, we  write $\|Z\|_{\psi_2}$ to denote the sub-Gaussian norm of $\nu$. We note that this definition applies also to conditional densities, and we write $\|Z\vert X\|_{\psi_2}$ to denote the sub-Gaussian norm of the conditional distribution of $Z$ given $X$.

\medskip
Throughout the paper the notation $a\lesssim b$ means that $a\leq c b$, for some constant $c$ that does not depend on the sample size $n$.

\subsubsection{Vectorization}
Let $\{h_W,\;W\in\W\}$ denote a generic deep neural network class of function where $h_W:\;\rset^{p_0}\to \rset^{p_D}$, with parameter $W=(W_D,\ldots,W_1) \in\W \eqdef \rset^{p_D \times p_{D-1}}\times \cdots\times \rset^{p_1 \times p_{0}}$. By vectorization, we will  view $\W$ as  the Euclidean space $\rset^{q}$ (where $q\eqdef \sum_{\ell=1}^D p_\ell p_{\ell-1}$), and we will use a generic notation $\|\cdot\|_2$ to denote its Euclidean norm. Similarly,  we will write $\|W\|_0$ (resp. $\|W\|_\infty$) to denote the number of non-zeros components of $W$ (resp. the largest absolute value  of the components of $W$). For any $1\leq \ell\leq D$, we will similarly view $W_\ell$ as a vector element of $\rset^{p_\ell p_{\ell-1}}$, and define similarly $\|W_\ell\|_2$, $\|W_\ell\|_0$ and $\|W_\ell\|_{\infty}$. Hence, in what follows, for a matrix $M$, $\|M\|_2$ will always denote the Frobenius norm of $M$, not its spectral norm. We will write the spectral norm as $\|\cdot\|_{\textsf{op}}$. 

\section{Learning to solve inverse problems}\label{sec:main:desc}
Summarizing the introductory discussion on the data generating process, we make the following assumption. 

\begin{assumption}\label{H1:dl}
We have a data set $\D = \{({\bf x}_i, {\bf y}_i),\;1\leq i\leq n\}$ of i.i.d. samples generated according to (\ref{data:model}) such that for  $1\leq i\leq n$,
\[{\bf x}_i = g({\bf y}_i) + \bm{\xi}_i,\;\;\;\mbox{ where }\;\;\;\PE(\bm{\xi}_i\;\vert \; {\bf y}_i)=0,\;\;\;\;\]
for some independent error terms $(\bm{\xi}_1,\ldots,\bm{\xi}_n)$. Furthermore we assume that each $\bm{\xi}_i$ is a conditionally sub-Gaussian random vector given 
${\bf y}_i$, with a non-random sub-Gaussian norm $\sigma_i<\infty$. 
\end{assumption}

%We stress again that  Assumption \ref{H1:dl} is an assumption on the data generating process, not on the posterior distribution  in (\ref{class:BIP}). Indeed in general, the posterior distribution $\pi_{\mu_0}(\cdot \vert {\bf y}_i)$ can be substantially different from  $\pi(\cdot\vert {\bf y}_i)$ since the marginal distribution $\mu$ is typically unknown, and we typically do not have enough data for the posterior distribution to possess any significant  contraction properties. In fact, one of the key challenge in  solving ill-posed inverse problems is precisely constructing prior distributions that approximate as best as possible the  true marginal distribution $\mu$.

\begin{remark}\label{rem:assump:H1dl}
Assumption H\ref{H1:dl} formalizes the discussion in the introduction on the concentration of the conditional distribution of ${\bf x}_i$ given ${\bf y}_i$. As expanded upon in the introduction, this assumption is conceptually justified in seetings where the latent variable ${\bf x}$ has additional structures, and the marginal distribution of ${\bf x}$ is concentrated on a low-dimensional subset of $\rset^{d_x}$. Checking H\ref{H1:dl} is similar to establishing a Bernstein-von Mises theorem  for the conditional distribution of ${\bf x}_i$ given ${\bf y}_i$, a  challenging problem that is beyond the scope of this work (\cite{vdv:98,bk:12,nickl:17,nickl:jakob:17}).  
%\vspace{-0.7cm}
%\begin{flushright}
%$\square$
%\end{flushright}
\end{remark}

%\begin{remark}
%In Assumption \ref{H1:dl} we are making the important technical assumption that the conditional sub-Gaussian norms $\sigma_1,\ldots,\sigma_n$ are non-random. Allowing $\sigma_i$ to depend on ${\bf y}_i$ would complicate the analysis without fundamentally changing the conclusions. Furthermore, in the applications it is common practice to normalize image pixel data to take values in $[0,1]$. The non-randomness of $\sigma_i$ becomes a reasonable  assumption is such cases.
%\end{remark}
%
%\medskip

 Let $\varsigma_i$ denote the conditional sub-Gaussian norm of $\|\bm{\xi}_i\|_2$ given ${\bf y}_i$.  The conditional sub-Gaussian assumption on $\bm{\xi}_i$ imposed in Assumption \ref{H1:dl} implies that $\varsigma_i<\infty$ (see e.g. Theorem 3.1.1 of \cite{vershynin:18}).  Throughout we set
 \[\bar\sigma \eqdef\max_{1\leq i\leq n}\sigma_i,\;\;\mbox{ and }\;\; \bar\varsigma \eqdef\max_{1\leq i\leq n}\varsigma_i.\]

 \subsection{Gradient descent networks}
We consider the  nonparametric regression (\ref{au:model}), where $\{g_W,\;W\in\W\}$ is a gradient descent network (GDN) function class  that we now define. First we introduce a generic feed-forward deep neural network function $H_W:\;\rset^{d_x}\to\rset^{d_x}$. Let $D>0$ be the depth of the network. Let $(p_D,\ldots,p_{0})$ be a sequence of integers representing the sizes of the layers of the network, with $p_0=d_x$, and $p_D=d_x$. For $1\leq \ell\leq D$, let $\mathsf{a}_\ell:\;\rset^{p_\ell}\to \rset^{p_\ell}$ be activation functions that we assume Lipschitz: for all ${\bf z}_1,{\bf z}_2\in\rset^{p_\ell}$,
\begin{equation}\label{stab:act:fun}
\mathsf{a}_\ell({\bf 0}) = {\bf 0},\;\;\mbox{ and }\;\; \|\mathsf{a}_\ell({\bf z}_1) - \mathsf{a}_\ell({\bf z}_2)\|_2 \leq \|{\bf z}_1 - {\bf z}_2\|_2.
\end{equation}
%We shall assume that $\mathsf{a}_D({\bf x}) = {\bf x}$, and for $1\leq \ell \leq D-1$,  each component of the activation function $\mathsf{a}_\ell$ is a $\textsf{ReLu}$ function (where $\textsf{ReLu}(x)=\max(x,0)$). 
For  $B\in\rset^{p_\ell\times p_{\ell-1}}$, we set
\begin{equation}\label{psi:fun}
\Psi_{B}^{(\ell)}({\bf z}) \eqdef \mathsf{a}_\ell(B {\bf z}),\;\;\;{\bf z}\in\rset^{p_{\ell-1}}.\end{equation}
With parameter $W= (W_D,\ldots,W_1)$, where $W_\ell\in\rset^{p_\ell\times p_{\ell-1}}$,  we consider the function $H_W:\rset^{d_x}\to \rset^{d_x}$ defined as
\begin{equation}\label{H:fun}
H_W({\bf x}) = \Psi_{W_{D}}^{(D)}\circ \cdots  \circ \Psi_{W_1}^{(1)} ({\bf x}),\;\;\;{\bf x}\in\rset^{d_x},
\end{equation}
where $f\circ g$ is the composition of $f$ with $g$.
%\begin{equation}\label{def:g:generic}
%g_W({\bf y}) \eqdef \Psi_{W_{D}}^{(D)}\circ \cdots  \circ \Psi_{W_1}^{(1)} ({\bf y}),\;\;\;{\bf y}\in\rset^{d_y}.\end{equation}
%The parameter space in  this case  is thus  $\W \eqdef \rset^{p_D \times p_{D-1}}\times \cdots\times \rset^{p_1 \times p_{0}}$.
\begin{remark}
Feed-forward deep neural network models are usually written with additional bias terms (that is, by defining $\Psi_{B}^{(\ell)}({\bf z})$ as  $\mathsf{a}_\ell(B {\bf z} + {\bf b})$). However our formulation incurs no loss of generality, since these bias parameters can always be subsumed into the matrix $B$,  by appropriately enlarging  $B$ and adding an intercept to the input.  
\end{remark}

%We assume that the activation function $\mathsf{a}_D$ of the last layer is taken such that for all $W\in\W$, and ${\bf x}\in\rset^{d_x}$,
%\begin{equation}\label{stab:last:layer}
%\|H_W({\bf x})\|_2 \leq \|{\bf x}\|_2.
%\end{equation}
%This can be achieved with an additional layer normalization step (see e.g. \cite{layer:norm:16}). 
Given $\gamma>0$, we use the function $H_W$ to approximate the proximal map of $\gamma\R$ (where $\R$ is as in (\ref{data:model})), defined as
\[\Prox^{\gamma \R}({\bf x}) \eqdef \argmin_{{\bf u}\in\rset^{d_x}} \left[ \gamma\R({\bf u}) + \frac{1}{2}\|{\bf u} - {\bf x}\|_2^2\right].\] 
 Given a step-size $\gamma>0$, $W\in\W$, and  ${\bf y}\in\rset^{d_y}$ we thus define the function $F_{{\bf y},W}:\;\rset^{d_x}\to\rset^{d_x}$ by
\[F_{{\bf y},W}({\bf x}) \eqdef H_W\left({\bf x} -\gamma \nabla_{\bf x} f({\bf y}\vert {\bf x}) \right).\]
%where 
%\begin{equation}\label{H:fun}
%H_W({\bf x}\vert {\bf y}) \eqdef \nabla_{\bf x} f({\bf y}\vert {\bf x}) + \nabla h_W({\bf x}),
%\end{equation}
% denotes the vector of partial derivative of $f({\bf y}\vert {\bf x}) + h_W({\bf x})$ with respect to ${\bf x}$. 
Given $D'\geq 1$ (the depth of the network), we consider the function $g_W$ defined as 
\begin{equation}\label{fun:dnn}
g_{ W }({\bf y}) \eqdef \underbrace{F_{{\bf y},W}\circ\cdots\circ F_{{\bf y},W}}_{D' \mbox{ times }}({\bf x}^{(0)}),
\end{equation}
for some initial value ${\bf x}^{(0)}\in\rset^{d_x}$. We note that in addition to $W$, the function $g_W$ depends also on the step-size $\gamma$, the depth $D'$, and the initial value ${\bf x}^{(0)}$.  The  network architecture in (\ref{fun:dnn}) is the so-called (proximal) gradient descent network (GDN), and belong to the class of so-called algorithm unrolling (or unfolding) deep learning models, where a statistical model is built by iterating an optimization algorithm. Many variations have been proposed in the literature based on various other optimization schemes (we refer the reader to the references in the introduction).   
%For instance, a slight modification known as Neumann networks (\cite{gilton:etal:20}), that mimics the Polyak-Ruppert averaging device (\cite{polyak:juditsky:92}), and sometimes works better in the application is defined as follows.
%\begin{equation}\label{fun:dnn:2}
%\tilde g_{ W }({\bf y}) \eqdef \frac{1}{D'}\sum_{\ell =1}^{D'} {\bf x}^{(\ell)},\;\; \mbox{ where }\;\;  {\bf x}^{(\ell)} = F_{{\bf y},W}( {\bf x}^{(\ell-1)}).
%\end{equation}

 For ${\bf x}\in\rset^{d_x}$, and ${\bf y}\in\rset^{d_y}$, we set
 \[F_{\bf y}({\bf x}) \eqdef \Prox^{\gamma \R}\left({\bf x} - \gamma \nabla_{\bf x} f({\bf y}\vert{\bf x})\right).\]
 Looking at the definition of (\ref{fun:dnn}), it is clear that under appropriate convexity assumptions and for well-selected step size $\gamma$, the convergence of $ F_{\bf y}^j({\bf x})$ toward $g({\bf y})$ is guaranteed, where $h^j$ denotes the composition of $h$, $j$ times. Therefore, if for some $W$, $H_W\approx \Prox^{\gamma \R}$, then we can expect $g_W\approx g$ for $D'$ sufficiently large, by standard convex optimization theory. As a result, the function class $\{g_W,\;W\in\W\}$ typically has good  skills in approximating $g$. 
We impose next the necessary assumptions for the intuition above to hold.

  \begin{assumption}\label{H2:dl}
  \begin{enumerate}
  \item The function $\R:\;\rset^{d_x}\to \rset$ is convex. There exists $M$ such that for all ${\bf y}\in\rset^{d_y}$, the function ${\bf x}\mapsto f({\bf y}\vert {\bf x})$ is convex, differentiable, and  with a $M$-Lipschitz gradient. Furthermore, the step-size $\gamma$ satisfies $0<\gamma \leq M^{-1}$, and for all ${\bf y}\in\rset^{d_y}$, $g({\bf y})$ is uniquely defined.
\item There exist $R_0<\infty$,  $\varrho_n\in [0,1)$  such that for all  $k\geq 1$,
 \[\max_{1\leq i\leq n}\; \|F_{{\bf y}_i}^k({\bf x}^{(0)}) - g({\bf y}_i)\|_2\leq R_0 \varrho_n^k.\]
%\item There exist finite $R_0^{'}$,  such that for all ${\bf y}\in\Yset$,
%\[
%\sup_{j\geq 0}\;\sup_{W\in\W}\|F_{{\bf y},W}^j({\bf x}^{(0)})\|_2 \leq R_0^{'}.\]
\end{enumerate}
  \end{assumption}

 \begin{remark}\label{rem:assump:H2dl}
Assumption H\ref{H2:dl}-(1) is a standard set up for proximal gradient descent (\cite{parikh:boyd:2013}). Assumption H\ref{H2:dl}-(2)  is stronger and imposes a linear convergence rate. It is well-known to hold for strongly convex problems. For problems where a local linear convergence holds,  H\ref{H2:dl}-(2)  can also be shown to follow from H\ref{H2:dl}-(1) when ${\bf x}^{(0)}$ is close enough to the solution. For instance it is known that  such local linear convergence of proximal gradient descent holds for the lasso problem (\cite{tao:etal:16}). The main challenge to go beyond the linear rate is the fact that sub-linear rates are typically expressed in terms of the function value, not in terms of the parameter as needed here.
\end{remark}

We also impose the following assumption that models the approximation of the proximal map $\Prox^{\gamma\R}$.
 
 \begin{assumption}\label{H3:dl}
 There exist $\beta_1,\beta_2\geq 0$, such that for all $\epsilon \in (0,1)$ we can construct a 
feed-forward deep neural network $H_W$, as in (\ref{H:fun}), with depth $1\leq D \leq  D_0\log(\sqrt{d_x}/\epsilon)$, maximum layer size no larger than $N_0 \left(\sqrt{d_x}/\epsilon\right)^{\beta_1}$, maximum parameter absolute value $\|W\|_\infty $ no larger than $1$, and maximum sparsity $\|W\|_0$ no larger than $ s_0\left(\sqrt{d_x}/\epsilon\right)^{\beta_2}$, for constants $D_0,N_0,s_0$ that do not depend on $\epsilon$ such that for all $R<\infty$,
\begin{equation}\label{eq:H3}
\sup_{{\bf x}:\;\|{\bf x}\|_2\leq R}\; \left\|H_W({\bf x}) - \Prox^{\gamma \R}({\bf x})\right\|_2 \leq \epsilon.\end{equation}
Furthermore, there exists $R_1<\infty$ such that with the constructed network $H_W$,
\begin{equation}\label{eq:H3:2}
\max_{j\geq 1}\;\max_{1\leq i\leq n} \|F_{{\bf y}_i,W}^j({\bf x}^{(0)})\|_2\leq R_1.\end{equation}
\end{assumption}

 \begin{remark}\label{rem:assump:H3dl}
Since ${\bf x}\mapsto \Prox^{\gamma \R}({\bf x})$ is a Lipschitz map, we can always invoke classical deep learning approximation theory for smooth functions (see e.g. \cite{shieber:20,devore:etal:21}) to conclude that Assumption \ref{H3:dl} holds with $\beta_1=\beta_2=d_x$. However better approximation is possible if $\Prox^{\gamma \R}$ is a simple map.

The condition in (\ref{eq:H3:2}) is a technical assumption that simplifies the mathematical analysis. It can be automatically enforced by  adding a layer-normalization layer   in  $H_W$ (\cite{layer:norm:16}).
\end{remark}

 \subsection{Bayesian inference using spike-and-slab priors}
We consider the problem of fitting model (\ref{au:model}), where $\{g_W,\;W\in\W\}$ is the GDN  function class constructed in (\ref{fun:dnn}). The parameter space is $\W \eqdef \rset^{p_D \times p_{D-1}}\times \cdots\times \rset^{p_1 \times p_{0}}$. As indicated at the end of the introduction, at times we shall view $\W$ as the Euclidean space $\rset^{q}$, where
\[q\eqdef \sum_{\ell=1}^D (p_\ell\times p_{\ell-1}).\]
Our initial motivation in this work comes from inverse problems in remote sensing. It was therefore important for us to analyze a statistical procedure that can be implemented in practice. An important shortcoming of the current statistical theory of deep learning models under sparsity constraints (\cite{barron:18,shieber:20,taheri:etal:21,e:etal:20}) is the lack of computational tractability of the resulting estimators.  To address this issue we propose to fit the model  $\{g_{ W },\; W \in\W\}$ in a Bayesian framework  using a  spike and slab prior (\cite{AB:19}).  To that end, we introduce a sparsity structure parameter $\Lambda=(\Lambda_D,\ldots,\Lambda_1) \in\mathcal{S} \eqdef \{0,1\}^{p_D \times p_{D-1}}\times \cdots\times \{0,1\}^{p_1 \times p_{0}}$ that encodes the support of $W$. We assume that $\Lambda$ has a prior distribution given by
\begin{equation}\label{prior:Lambda}
\Pi_0(\Lambda) \propto \left(\frac{1}{q}\right)^{(\mathsf{u}+1)\|\Lambda\|_0},\;\;\;\Lambda\in\mathcal{S},
\end{equation}
for some parameter $\mathsf{u}\geq 1$. This prior corresponds to the assumption that the entries of $\Lambda$ are independent Bernoulli random variables $\textbf{Ber}((1+q^{\mathsf{u}+1})^{-1})$. Given $ \Lambda $ we assume that the entries of $W$ are conditionally independent with joint density
\begin{multline}\label{prior:W:1}
\Pi_0(W\vert  \Lambda ) \; = \; \prod_{\ell=1}^D\;\prod_{(i,k):\;\Lambda_{\ell,i,k}=1} \;\sqrt{\frac{\rho_1}{2\pi}} e^{-\frac{\rho_1}{2} W^2_{\ell,i k}}  \prod_{(i,k):\;\Lambda_{\ell,k}=0} \sqrt{\frac{\rho_0}{2\pi}}e^{-\frac{\rho_0}{2}W_{\ell,i,k}^2},
\end{multline}
for some parameters $0<\rho_1<\rho_0$. Throughout the paper, and without further notice we set
\begin{equation}\label{rho1_val}
\rho_1=1.
\end{equation}
The variance parameter $\rho_0$ can be chosen fairly arbitrarily. However, in order to ease MCMC sampling from the resulting posterior distribution it is crucial to choose $\rho_0$ small, of order $1/n$. We refer the reader to (\cite{AB:19}) for further discussion. 
%
%\begin{assumption}\label{H3:dl}
%The prior distribution $\Pi_0$ is as in (\ref{prior:Lambda}), (\ref{prior:W:1}), and (\ref{prior:W:2}). Furthermore, if $(\Lambda,W)\sim\Pi_0$ then $cW\in\W$ $\Pi_0$-almost surely.
%\end{assumption}
Using this prior distribution and the regression model (\ref{au:model}), we consider the posterior  distribution on $\Theta\eqdef \mathcal{S}\times \W$ with density  given by
\begin{equation}\label{inverse:glut:post}
\Pi( \Lambda ,  W \;\vert\;\D)\propto \Pi_0( \Lambda ,  W ) \exp\left(-\frac{1}{2\sigma^2}\sum_{i=1}^n \|{\bf x}_i - g_{ W \odot \Lambda }({\bf y}_i)\|_2^2\right),\end{equation}
where $ W \odot \Lambda $ denotes the component-wise product of $ W $ and $ \Lambda $. To use this posterior distribution we draw sample $(\Lambda,W)\sim \Pi(\cdot\vert\D)$, and use $g_{ \Lambda \odot W }$ as inversion map. Since $\Lambda$ is typically sparse under $\Pi$, $g_{ \Lambda \odot W }$ is a sparse GDN.  For $h:\rset^{d_y}\to\rset^{d_x}$, we set
\[\|h\|_n\eqdef \sqrt{\frac{1}{n}\sum_{i=1}^n \|h({\bf y}_i)\|_2^2}.\]
%,\;\; \mbox{ and }\;\;\; \|h\|_\infty\eqdef \sup_{{\bf y}\in\Yset}\|h({\bf y})\|_2.\]
Our goal is to derive a bound on $\|g_{ \Lambda \odot W } - g\|_n$, when $(\Lambda,W)\sim \Pi(\cdot\vert\D)$. 
\begin{theorem}\label{thm:2}
Assume H\ref{H1:dl}-H\ref{H3:dl}. Consider the nonparametric regression (\ref{au:model})  for estimating $g$, where the function class $\{g_W,\;W\in\W\}$ is as defined in (\ref{fun:dnn}), and the regression variance parameter $\sigma$ satisfies $\sigma\geq \bar\sigma$. Then for all $q$ large enough, and $n\geq \sigma^2\log(p)$, we can construct a function class $\{H_W,\;W\in\W\}$, such that at unrolling depth $D'$ that satisfies
\[D'\gtrsim \frac{\log(n)}{-\log(\varrho_n)},\]
the posterior distribution $\Pi(\cdot\vert \D)$ in (\ref{inverse:glut:post}) satisfies
\begin{equation}\label{eq:thm:2}
\Pi\left(\|g_{\Lambda\odot W} - g\|_n > M\bar\sigma \frac{\left(D'\right)^{1+\frac{\beta_2}{2}}}{n^{\frac{1}{2+\beta_2}}}  \; \vert \; \D \right) \leq \frac{12}{q},\end{equation}
with probability at least $1 - e^{-c_1n}-\frac{c_1}{q}$, for some absolute constant $c_1$, and a constant $M\lesssim (\log(q))^{1/(2+\beta_2)}\log(n)^{3/2}$.
\end{theorem}
\begin{proof}
See Section \ref{sec:proof:thm:2}.
\end{proof}

We make several remarks here.  (a) In contrast to common practice where $D'$ is often chosen on an ad-hoc manner, Theorem \ref{thm:2} recommends carefully scaling  the depth parameter $D'$ as 
\[D'\sim-\log(n)/\log(\varrho_n),\]
for optimal performance.  (b) The expression of the rate in (\ref{eq:thm:2}) suggests that the statistical performance of GDN unrolled at depth $D'$ deteriorates as $D'$ increases, implying an overfitting phenomenon. Although we do not have a matching lower bound theory to confirm this overfitting phenomenon, we have performed  several numerical experiments that  all show an overfitting of the model as $D'$ increases. (c) Algorithm unrolling allows researchers to build deep neural network architectures that exploit the structure of the problem. Are those architecture provably better than off-the-shelves architectures that do not make use of the forward problem? Our results shed some light on this question. In the setting of H\ref{H2:dl}, the function $g$ of interest is  at best Lipschitz\footnote{Indeed, it can be easily shown that if for all ${\bf y}\in\Yset$, ${\bf x}\mapsto f({\bf y}\vert {\bf x})$ is strongly convex with strong convexity parameter $\underline{\kappa}$, and ${\bf x}\mapsto \nabla_{{\bf x}}f({\bf y}\vert {\bf x})$ is $\bar\kappa$ Lipschitz then ${\bf y}\mapsto g({\bf y})$ is $L_g$-Lipschitz with $L_g \leq 2\bar\kappa/\underline{\kappa}$}. Therefore the minimax rate in the estimation of $g$ in a nonparametric regression setting without further knowledge on the structure of the problem is 
\[ C_2n^{-\frac{1}{2+d_y}}.\]
We  can invoke classical deep learning approximation theory (see e.g. \cite{yarostky:17,shieber:20,devore:etal:21})   to conclude that H\ref{H3:dl} holds with $\beta_1=\beta_2=d_x$. In that case, up to log-terms, we deduce from Theorem \ref{thm:2} that GDN achieves the convergence rate 
\[C_1 n^{-\frac{1}{2+d_x}}.\]

Hence, Theorem \ref{thm:2}  implies that in inverse problems where $d_y$ is larger than $d_x$, the unrolling framework has a  better convergence rate than the minimax rate of estimating $g$ from the data $\D$ in a nonparametric regression. However Theorem \ref{thm:2}  has some limitations. Firstly, the constants $C_1,C_2$ in the rates posted above depend on $d_x$ and $d_y$ in ways that are poorly understood.  This comes from the scalings of constants in current deep neural network approximation theory \cite{yarostky:17,shieber:20}. %$M$ depends also on the output dimension $d_x$ through the uniform and  Lipschitz norms of  $g$. And in the opposite direction,  the maximum sub-Gaussian norm $\bar\sigma$ typically becomes smaller as $d_y$ increases, and same for the Lipschitz constant of $g$. Clearly more research is needed to better understand the behavior of deep learning models in high-dimensions.
Another limitation of current minimax rates is the fact that deep learning models can often adapt to additional properties of the function of interest and converge much faster than the theoretical minimax rate. For  instance \cite{shieber:20} shows that FNN models achieves faster rate in the estimation of compositional functions. We give a similar example below.

(d) The use of the empirical norm $\|u\|_n = \sqrt{\sum_{i=1}^n u({\bf y}_i)^2}$ instead of the $L^2$ population norm of ${\bf y}$ in (\ref{eq:thm:2}) is another limitation of our result, although this is a fairly common practice in nonparametric estimation, and does not fundamentally change the resulting contraction rate. More technically, working in the $L^2$ norm amounts to the additional control of the term 
\begin{equation}\label{rev:3:eq1}
\sup_{W\in\wtilde{W}^{(j)}} \left|n^{-1}\sum_{i=1}^n(g_W({\bf y}_i) - g({\bf y}_i))^2 - \|g_W - g\|_2^2\right|,\end{equation}
in Lemma D.5. Because the sup in (\ref{rev:3:eq1}) is taken over well behaved sets $\wtilde{W}^{(j)}$, this uniform deviation can be controlled using standard tools as in \cite{wainwright:19}~Chapter 14, but would require additional assumptions on the marginal distribution of ${\bf y}$ that we wish to avoid making. 
	
%Although we did not study the question, it is likely that such compositional nature improvement applies as well to FNN in the estimation of function $g$. 

%The result matches the posterior analysis of \cite{polson:18} built on spike-and-slab priors with point-mass at the origin, as well as \cite{shieber:20} and others in the frequentist setting. This convergence rate is slow. The actual convergence rate can be faster if $g$ possesses some additional properties that the model can adapt to. For  instance \cite{shieber:20} shows that FNN models achieves faster rate in the estimation of compositional functions. It is unclear whether such compositional nature applies to the function $g$. 

\subsubsection{Application to sparse marginal distributions}
We give another application of Theorem \ref{thm:2}  where the posterior predictive function obtained from the GDN achieves the parametric rate.  When dealing with images, several authors such as \cite{beck:teboulle:2010,dong:etal:11} have argued that natural image data are often sparse after linear transformation (such as a difference operators, or wavelet transforms), and  suggested modeling the  marginal distribution $\mu$ as 
\begin{equation}\label{ex:mu}
\mu(\rmd {\bf x}) = \frac{1}{c_\mu}e^{-\R_0(B{\bf x})} \rmd{\bf x},
\end{equation}
for some simple sparsity inducing function $\R_0$, and a non-singular matrix $B\in\rset^{d_x\times d_x}$. In other words, $\R({\bf x}) = \R_0(B{\bf x})$. A common choice is  $\R_0({\bf x}) = \lambda \|{\bf x}\|_1$ or $\R_0({\bf x}) = \lambda \|{\bf x}\|_1 + \lambda_2 \|{\bf x}\|_2/2$, for parameters $\lambda,\lambda_1,\lambda_2\geq 0$.  If $B$ is an orthogonal matrix, and $\Prox^{\gamma \R_0}$ denotes the proximal operator of $\R_0$, then by proximal calculus (see e.g. Lemma 2.8 of \cite{combettes:wajs:2005}), we have
\begin{equation}\label{eq:prox}
\Prox^{\gamma \R}({\bf x}) = B^{-1} \Prox^{\gamma \R_0}\left(B {\bf x}\right).
\end{equation}
For example, given $\lambda_1>0,\lambda_2\geq 0$, suppose that $\R_0$ is the elastic-net regularization prior of \cite{elastc:net:05} given by
\begin{equation}\label{eq:en}
\R_0({\bf x}) = \lambda_1 \|{\bf x}\|_1 + \frac{\lambda_2}{2} \|{\bf x}\|_2^2.\end{equation}
Then the proximal of $\gamma\R_0$ is $\Prox^{\gamma \R_0}({\bf x}) = (\mathsf{s}_\gamma(x_1),\cdots\mathsf{s}_\gamma(x_{d_x}))^{\texttt{T}}$, where
\[\mathsf{s}_\gamma(x) = \mathsf{ReLu}\left(\frac{x-\gamma\lambda_1}{1+\gamma\lambda_2}\right) - \mathsf{ReLu}\left(\frac{-x-\gamma\lambda_1}{1+\gamma\lambda_2}\right),\]
and where $\mathsf{ReLu}(t)\eqdef\max(t,0)$. Therefore, $\Prox^{\gamma \R_0}({\bf x})$ can be represented exactly using a 2-layer  $\mathsf{ReLu}$ neural network with layer sizes $(d_x,2d_x,d_x)$, and $\Prox^{\gamma \R}({\bf x})$ can be represented exactly using a 4-layer  $\mathsf{ReLu}$ neural network with layer sizes $(d_x,d_x,2d_x,d_x,d_x)$. Hence, H\ref{H3:dl} holds with depth $D=4$, $\beta_1=\beta_2=0$. Furthermore, since $\R$ is strongly convex, if we focus on the linear regression model and take the forward model as in (\ref{f:model:lin}),  then H\ref{H2:dl} holds. Hence Theorem \ref{thm:2} yields the following.

\begin{corollary}\label{coro:1}
Suppose that H\ref{H1:dl} holds with $f$ as in (\ref{f:model:lin}), and suppose that $\mu$ is as in (\ref{ex:mu}) with some orthogonal matrix $B$, and $\R_0$ as in (\ref{eq:en}). Suppose also that $\sigma\geq \bar\sigma$. Then we can construct a deep learning function class $\{H_W,\;W\in\W\}$, with depth $D=4$, such that at unrolling depth $D'\gtrsim -\log(n)/\log(\varrho_n)$ the posterior distribution $\Pi(\cdot\vert \D)$ in (\ref{inverse:glut:post}) satisfies 
\[\Pi\left(\|g_{\Lambda\odot W} - g\|_n \geq \frac{M\bar\sigma D'}{\sqrt{n}} \vert \; \D \right) \leq \frac{12}{q},\]
with probability at least $1 - \frac{c_1}{q} -  e^{-c_1n}$, for some absolute constant $c_1$, where $M$ depends on some log terms that we ignore.
\end{corollary}

\section{Numerical illustration}\label{sec:illust}
We illustrate our theoretical results with a toy  example, a simulation and a real data deblurring problem. For all  examples we draw  samples from the posterior distribution  (\ref{inverse:glut:post})  using the Sparse Asynchronous Stochastic Gradient Langevin Dynamics (SA-SGLD) sampler of \cite{atchade:wang:21}, an approximate MCMC sampler designed  for posterior distributions of the form (\ref{inverse:glut:post}), that employs asynchronicity for fast sampling. For more details on the approximate correctness of the sampler we refer the reader to \cite{atchade:wang:21}. Computationally the SA-SGLD sampler is implemented at the cost of 2 back-propagation through the GDN per MCMC iteration. %In the algorithm we alternate between an update of $W$ using the Stochastic Gradient Langevin Dynamics of \cite{sgld} while keeping $\Lambda$ fixed, and an update of $\Lambda$ with $W$ fixed. The update of $\Lambda$ is performed using a variation of the asynchronous Gibbs sampler of \cite{desa:16}. For more details on the approximate correctness of the sampler we refer the reader to \cite{atchade:wang:21}. Computationally the SA-SGLD sampler is implemented at the cost of 2 back-propagation through the GDN per MCMC iteration.%By deep learning standards the illustrations reported here are small-scale examples, designed mainly to illustrate our main theoretical results. Larger scale implementations of GDN and its thorough comparison with other architectures can be found for instance  in(\cite{}).

\subsection{Learning the Elastic Net regression map}\label{sec:ex1}
In this section, we illustrate our theoretical results with the example of learning the elastic-net regularization to solve a linear regression model. %We implement two experiments to support our theoretical results.

\textbf{Data generation: }\textbf{Data generation: }We generate a dataset $\D = \{({\bf x}_i,{\bf y}_i),\;1\leq i\leq n\}$  where ${\bf x}_i\stackrel{i.i.d.}{\sim}\mu$, and ${\bf y}_i\vert {\bf x}_i\sim\textbf{N}(A{\bf x}_i,v^2\mathbf{I}_{d_y})$. The entries of the matrix $A$ are generated independently from the standard normal distribution, and we set $v^2=0.001$. We choose $\mu(\rmd {\bf x})\propto e^{-\R_0(B{\bf x})} \rmd{\bf x}$ as in (\ref{ex:mu}), where $\R_0$ is the elastic net density as in (\ref{eq:en}), and $B=\mathbf{I}_{d_x}$.  We set  $n=200$, $d_x =d_y = 100$, and $\lambda_1=1$, $\lambda_2=1$. %We refer the reader to Appendix \ref{sec:detail:ex1} for further details on the data generation.%In the first simulation, we generate $\D_1=\{({\bf x}_i,{\bf y}_i),\;1\leq i\leq n\}$  as in (\ref{data:model}) with $n=200$, where $f$ is the linear model in (\ref{f:model:lin}), and $\R$ is the elastic-net regularization $\R_0$ given in (\ref{eq:en}). To make this example more practical, we implement the second simulation with the input data as a product of a non-singular matrix $B\in\rset^{d_x\times d_x}$ with ${\bf x}$. The data is generated as $\D_2=\{(B{\bf x}_i,{\bf y}_i),\;1\leq i\leq n\}$ with $n=200$ where $B{\bf x}_i$ is from the elastic-net $\R({\bf x}) = \R_0(B{\bf x})$ , and ${\bf y}_i$ is from the linear model in (\ref{data:model}). In both simulations, we have  $d_x =d_y = 100$, and $\lambda_1=1$, $\lambda_2=1$.  

\textbf{Model architecture}. We specify $H_W$ as a FNN with depth $D=2$, and layer $(p_0,p_1,p_2)= (d_x, 2d_x, d_x)$. To prevent overfitting, we specify the models to only learn the parameters that connect the nodes between layers in each dimension of ${\bf x}_i$, which reduces the number of model parameters to $q = 7*d_x = 700$. We consider several values of the unrolling depth $D'$: $D'=1$ (GDN1), $D'=5$ (GDN2), $D'=10$ (GDN3), and $D'=20$ (GDN4) for comparison\footnote{The step-size $\gamma$ of GDN is taken as $\gamma_1 = \frac{2 v^2}{\lambda_{\textsf{max}}(A'A)}$ a, where $\lambda_{\textsf{max}}(A'A)$ is the largest eigenvalue of $A'A$.}.  

\textbf{Training details:} For the Bayesian prior we choose $\rho_0=n$, $\rho_1 =1$, and $\mathsf{u}=1$. We choose $\sigma = 0.001$ in (\ref{inverse:glut:post}), and run the SA-SGLD with a constant step-size $2*10^{-6}$ for GDN. The mini-batch size is set to $100$. The initial value ${\bf x}^{(0)}$ of GDN is set to $0$. The MCMC sampler is implemented in \textsf{Pytorch} on a high-performance computer with a Nvidia Tesla V100 GPU. We run the sampler for $10^4$ iterations. 

\textbf{Evaluation procedure:} For the comparison, we generate $1000$ test samples and evaluate the prediction errors of the resulting estimator $g_{\Lambda\odot W}$, where $(\Lambda,W)\sim \Pi(\cdot\vert\D)$. The performance of GDN are compared to the performance of $g$, which in this problem is easily calculated by proximal gradient descent. We do this comparison by computing the error
\[e(\Lambda,W) =\frac{1}{1000}\sum_{i=1}^{1000} \|g_{\Lambda\odot W}({\bf y}_i) - g({\bf y}_i)\|_2,\]
where the average is taking over the test sample.  The boxplots in Figure \ref{fig:en:loss} show the distributions of the errors obtained by taking $500$ samples of $(\Lambda,W) $ along the MCMC sampler. In this toy example, $D'=10$ (GDN3)  yields the best results and as predicted by our theory, GDN deteriorates as the unrolling depth increases.

\medskip
\begin{figure}[h]
	\includegraphics[scale=0.85]{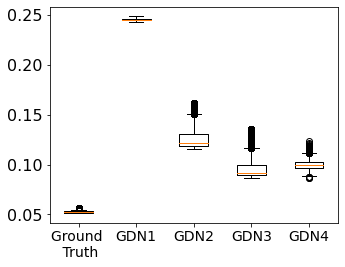}
	\caption{Distributions of prediction errors on test samples: GDN1 (D'=1), GDN2 (D'=5) and GDN3 (D'=10), GDN4 (D'=20).
	} 
	\label{fig:en:loss}
\end{figure}

\medskip

\subsection{Illustration with a simulated data deblurring problem}\label{sec:ex2}
Image deblurring is a common inverse problem in computational imaging. Here for illustration, we start with a simulated dataset $\D$ for experimental purposes.

\textbf{Data generation. }We generate $n=500$  images of size $16\times 16$,  that we then blurred using a Gaussian blurring convolution kernel with variance $3$ without restriction on the kernel range\footnote{Each image is $16 \times 16$ matrix partitioned into 4 blocks where the upper left and lower right are $8\times 8$ diagonal matrices with diagonal elements sampled from $N(20, 0.5)$ and $N(-10, 0.1)$ respectively, where the upper right is a $8\times 8$ matrix with entries sampled from $N(10, 0.1)$, and the lower left a $8\times 8$ matrix with entries sampled from $N(-10, 5)$.}. 

\textbf{Model architecture.} We construct $H_W$ in (\ref{H:fun}) as a 3-layer relu-convolutional neural network\footnote{consisting of 3 convolutional layers with respective sizes $3,3,1$, respective number of filters $32,64,1$. Each layer except the last layer is followed by a \textsf{LayerNorm} layer and a \textsf{ReLu} layer}. The total number of parameters is $q=18,881$. We evaluate the GDN model at unrolling depth $D' = 2$ (GDN1), $D' = 4$ (GDN2), $D'=12$ (GDN3) and $D'=24$ (GDN4), and we do a comparison with two feedforward convolutional neural network (FNN) that do not make use of the forward problem. The first FNN has the same  architecture as $H_W$  (FNN1), while the second is a 6-layer FNN with  total number of parameter $q = 136,641$ (FNN2)\footnote{The 6-layer network consists of 2 convolutional layers, 1 channel-wise fully connected layer, 3 deconvolutional layers with respective sizes $5, 3, 2, 4, 5, 3$, respective number of filters $32, 64, 64, 64, 32, 1$. Each layer except the last layer is followed by a \textsf{LayerNorm} layer and a \textsf{ReLu} layer.}. All the layers are padded to keep the image size constant.

\textbf{Training details: } For the Bayesian prior, we use $\rho_0=n, \rho_1 =1$, and $\mathsf{u}=8000$. We choose $\sigma^2 = 0.01$ in (\ref{inverse:glut:post}), and run the SA-SAGLD with a constant step-size $2 \times 10^{-8}$ for both FNN and GDN. The mini-batch size is set to 50 in both cases. The MCMC sampler are implemented on a Nvidia Tesla V100 GPU system with 384 GB GPU memory running \textsf{Pytorch}. We run both samplers for $10^4$ iterations.

\textbf{Evaluation procedure. }
We generate 500 test samples to evaluate the prediction errors of the six models. The boxplots in Figure \ref{fig:ex2:bp} show the distribution of the mean square error (same as \ref{sec:ex1}) of the last 2000 samples of $(\Lambda, W)$ along the MCMC sampler of each model. Figure \ref{fig:ex2:recon} shows an example of  reconstruction from FNN1, FNN2, GDN1, and GDN3. We observe that GDN outperforms FNN1, and can achieve similar performance as FNN2 when the unrolling depth is appropriately selected (not too small, nor too large). The experiment again confirms the importance of scaling appropriately the unrolling depth as highlighted in our theoretical results.

\begin{figure}[h!] 
\centering 
\includegraphics[scale = 0.55]{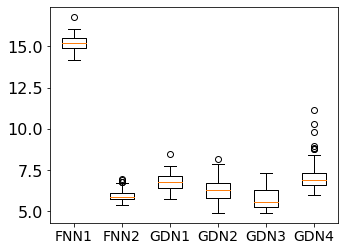}  
\caption{Test loss comparison between FNN1 (3 conv), FNN2 (3 conv + 1 cfc + 3 deconv), GDN1 ($D'=2$), GDN2 ($D'=4$), GDN3 ($D'=12$) and GDN4 ($D'=24$)} 
\label{fig:ex2:bp}
\end{figure}

\begin{figure}[h!] 
\centering 
\includegraphics[scale = 0.4]{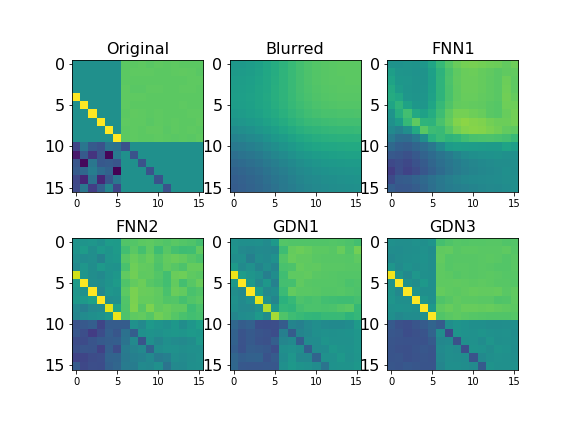}  
\caption{Reconstruction result from FNN1 (upper right), FNN2 (lower left), GDN1 $D'=2$ (lower middle) and GDN3 $D'=12$ (lower right)} 
\label{fig:ex2:recon}
\end{figure}

\subsection{Illustration with CelebA dataset}\label{ex:celebA}
We extend the last example to the deblurring of CelebA images \cite{celebA}.

\textbf{Data generation. }We randomly select $20,000$ images from the celebA dataset that we resize  to $64 \times 64$.  We generate the corresponding observed measurements ${\bf y}_i$ through the linear forward model (\ref{f:model}), where $A$ is a Gaussian blurring convolution matrix  with variance $6.25$, and where $v^2=0.01$ leading to a highly ill-conditioned inverse problem.  

\textbf{Model architecture. } We take $H_W$ in (\ref{H:fun}) as a 3-layer relu-convolutional neural network\footnote{with  kernel sizes $4,4,2$, and the number of filters that we take here as $256, 256, 1$. All padded to keep image size constant}. The depth of the GDN is either $D'=4$ (GDN1), $D'=12$ (GDN2), or $D'=24$ (GDN3). The total number of parameters in the same in all three cases and equal to $q=267,777$. We compare this model with a  feedforward architecture that doe not make use of the forward model\footnote{With 3 convolutional layers (with size $4,4,2$, filter number $256,256,32$, and strides $1,2,1$) followed by 4 corresponding deconvolutional layers (with kernel size $2,4,4,2$, filter number $32,256,256,1$, and strides $1,2,1,1$).}. The total number of parameter of the FNN is  $q=2,271,297$. 

\textbf{Training details: } In the MCMC, the mini-batch size taken as $B=164$, and the step-size is taken  as $\gamma=10^{-9}$ for FNN, GDN2, and GDN3, and $\gamma=10^{-8}$ for GDN1. We run the algorithms for $80,0000$ iterations using a high-performance computier with a Nvidia Tesla V100 GPU running \textsf{Matlab 2022a}.

\textbf{Evaluation procedure. }Figure \ref{fig:ex:celebA1} shows the distributions of the test error from $500$ drawn from the MCMC sampler after burn-in. We see again that at appropriate depth GDN matches FNN. However we see a decrease in performance at deeper depth $D'=24$, which again suggests overfitting. Figure \ref{fig:ex:celebA2} shows three examples of  reconstructed images.

\begin{figure}[h!] 
\centering 
\includegraphics[scale = 0.45]{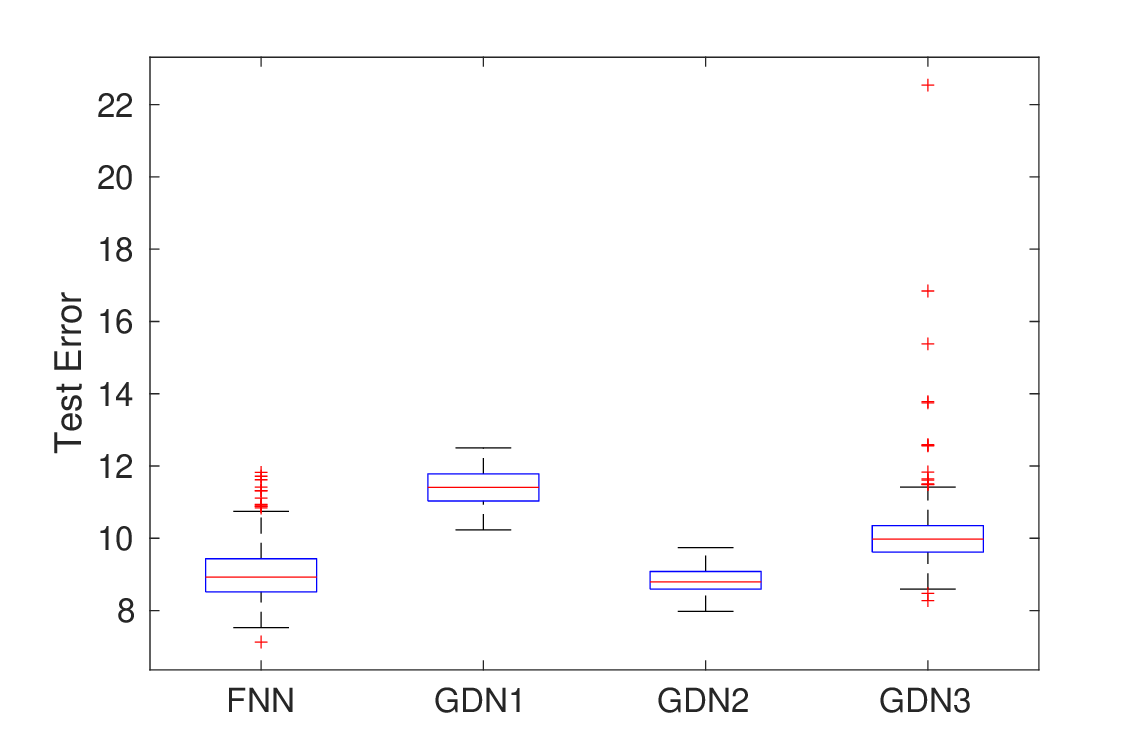}  
\caption{Test loss comparison between FNN (3 conv + 4 deconv), GDN1 (D'=4), GDN2 (D'=12) and GDN3 (D'=24)} 
\label{fig:ex:celebA1}
\end{figure}

\begin{figure}[h!] 
\centering 
\includegraphics[scale = 0.7]{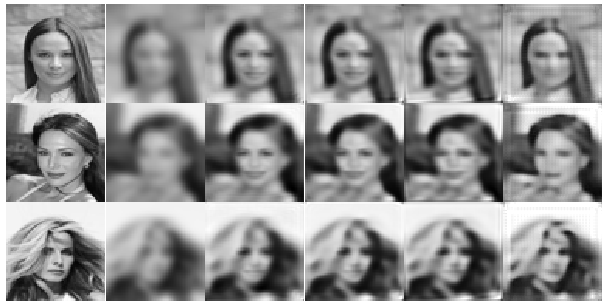}  
\caption{Some random examples of reconstructions. From left to right: CelebA image, blurred image, GDN1, GDN2, GDN3, and FNN} 
\label{fig:ex:celebA2}
\end{figure}

\section{A general Bayesian posterior contraction result}\label{sec:gen:post:contraction}
Theorem \ref{thm:2} is derived as special cases of a more general result of independent interest that we establish in this section. We consider again the regression model (\ref{au:model}), where  $\{g_W,\;W\in\W\}$ is some arbitrary  deep neural network function class.  We assume that the parameter space is $\W \eqdef \rset^{p_D \times p_{D-1}}\times \cdots\times \rset^{p_1 \times p_{0}}$, for some depth $D\geq 1$, and layer dimensions $p_0,p_1,\ldots,p_D\geq 1$.  As indicated at the end of the introduction, at times we shall view $\W$ as the Euclidean space $\rset^{q}$, with Euclidean norm denoted $\|\cdot\|_2$, where
\[q\eqdef \sum_{\ell=1}^D (p_\ell\times p_{\ell-1}).\]
We make the following local Lipschitz assumption on the function class.

\begin{assumption}\label{H5:dl}
For all  $0<\eta<\infty$, there exists  $L(\eta)\geq 1$ such that for all $W,W'\in\W$ that satisfy $\max(\|W\|_{2}, \|W'\|_{2}) \leq\eta$, and for all ${\bf y}\in\Yset$, we have
\begin{equation}\label{cond:grad}
\|g_W({\bf y}) - g_{W'}({\bf y})\|_2 \leq L(\eta)\|W - W'\|_2.
\end{equation}
\end{assumption}

The constant $L(\eta)$ is a local Lipschitz constant of the function $W\mapsto g_W({\bf y})$. Controlling appropriately these local Lipschitz constants is a major theoretical challenges in dealing with deep neural networks.  
%The next results shows that the contraction properties of $\Pi$ is determined by the usual two key factors: the approximation capacity of the family $\{g_W,\;W\in\W\}$, and its complexity, which is determined mainly by the function $L(\eta)$ in Assumption \ref{H3:dl}. 

%For $s,b\geq 0$, we set
%\[\W(s,b)\eqdef\left\{W\in\W:\; \|W\|_0\leq s,\; \mbox{ and }\;\;\|W\|_\infty\leq b\right\}.\]
 
\begin{theorem}\label{thm:main}
Suppose that the dataset $\D$ is generated as in H\ref{H1:dl}, and consider the nonparametric regression (\ref{au:model})  for some function class $\{g_W,\;W\in\W\}$ that satisfies H\ref{H5:dl}, and the corresponding posterior distribution (\ref{inverse:glut:post}).  Suppose that the regression variance parameter $\sigma$ satisfies $\sigma\geq \bar\sigma$. Let $\varpi_\star\geq 0$, $s_\star\geq 1$, be such that
\begin{equation*}
 \min\;\left\{\|g_{ W } - g\|_\infty,\; W \in\W\;\;\mbox{s.t.}\;\; \|W\|_{0} \leq s_\star,\;\;\|W\|_\infty\leq 1\right\} \leq \varpi_\star, \end{equation*}
and  set $\mathsf{L}_\star \eqdef L(2s_\star^{1/2})$, where the function $L$ is as in H\ref{H5:dl}. Define
\begin{equation*}
s \eqdef \left(1+ \frac{\log(\mathsf{L}_\star\sqrt{n})}{\mathsf{u}\log(q)}\right) s_\star +  \frac{4n\varpi_\star^2}{\sigma^2\mathsf{u}\log(q)},\;\mbox{ and }\; 
\r \eqdef \bar\sigma\sqrt{\frac{s\log(q) + s\log( \mathsf{L}_s)}{n}},\end{equation*}

%{\color{red}Rmk from QZ: in your proof you have
%\[s \eqdef  \frac{1}{\mathsf{u}} + \left(1+ \frac{5}{\mathsf{u}} + \frac{\log(\mathsf{L}_\star\sqrt{n})}{\mathsf{u}\log(q)}\right) s_\star +  \frac{4n\varpi_\star^2}{\sigma^2\mathsf{u}\log(q)},\;\mbox{ and }\;\; \r \eqdef \bar\sigma\sqrt{\frac{s\log(q \mathsf{L}_s)}{n}},\]}
where 
\[\mathsf{L}_s \eqdef L(s^{1/2}b_s),\;\;\mbox{with }\;\; b_s\eqdef \sqrt{2(1+\mathsf{u})(1+s)\log(q)}.\]
Then for all $q$ large enough, and $n\geq \sigma^2\log(q)$, we can find a constant $M^2\geq \mathsf{u}\max((\sigma/\bar\sigma)^2,1)$, and absolute constant $c_1$ such that
\begin{equation}\label{conclusion:thm1}
\Pi\left(\|g_{\Lambda\odot W} - g\|_n > M\r \;\vert \D\right) \leq \frac{12}{q},\end{equation}
with probability at  least $1-e^{-c_1 n} -\frac{c_1}{q}$.
\end{theorem}
\begin{proof}
See Section \ref{sec:proof:thm:main}.
\end{proof}
 
\begin{remark}
Theorem \ref{thm:main} applies well beyond the GDN of interest in this work. For any function class $\{g_W,\;W\in\W\}$  trained under the proposed sparse spike-and-slab prior, one can read off the posterior contraction rate of $\Pi(\cdot\vert\D)$ from Theorem \ref{thm:main}. The rate is driven by  the local Lipschitz constant $L(\eta)$  of the function class, and the relationship between $(s_\star,\beta_\star)$ and $\varpi_\star$, which captures the approximation capability of the function class.
\end{remark}

\subsection{Sketch of the proof of Theorem \ref{thm:main}}
To improve readability we give here a high-level description of the proof of Theorem \ref{thm:main}. Several approaches have been developed in the literature to study the contraction of posterior distributions.  Here  we follow an approach due to \cite{shen:wasserman:01}.  The merit of their approach is that it makes a direct connection between the contraction properties of the posterior distribution and the properties of the corresponding log-likelihood empirical process.

Let $f,\;\{f_\theta,\;\theta\in\Theta\}$ be a family of densities on a measurable space $\Zset$ equipped with a reference sigma-finite measure that we write as $\rmd z$. All densities considered on the sample space $\Zset$ are defined  with respect to $\rmd z$.  The parameter space $\Theta$ is some arbitrary measurable space.  Let $\pi$ be a prior probability measure on $\Theta$.  We consider the posterior distribution of $\theta$  given by
\[\Pi(A\vert z)  = \frac{\int_A f_\theta(z) \pi(\rmd \theta)}{\int_{\Theta}f_\theta(z) \pi(\rmd \theta)},\;\; A \mbox{ meas.},\;\;z\in\Zset.\]
The next lemma is a generalization of \cite{shen:wasserman:01}, and summarizes the main arguments used in the proof of  Theorem \ref{thm:main}.
\begin{lemma}\label{basic:lem}
Let $S,B$ and $\{\Xi_k,\;k\geq 1\}$ be measurable subsets of $\Theta$, such that $S\cap B^c\subseteq \cup_{k\geq 1} \Xi_k$.  Let $\beta>0$, $\rho\geq 0$ and $\{r_j,\;j\geq 1\}$ a sequence of positive numbers. Let  $\e$ be any subset of $\Zset$ such that
\begin{multline}\label{event:e}
\e\subseteq\left \{z\in \Zset:\; \int_{\Theta}\frac{f_\theta(z)}{f(z)} \pi(\rmd \theta) \geq e^{-\beta},\;\;\int_{S^c}\frac{f_\theta(z)}{f(z)} \pi(\rmd \theta) \leq \rho\right.\\
\left.\;\;\;\;\mbox{ and }\;\;\;\; \sup_{\theta\in \Xi_j}\; \left[ \log f_\theta(z) - \log f(z)\right] \leq -r_j \;\;\mbox{ for all }\;\; j\geq 1\right\}.
\end{multline}
Then for all $z\in\e$, we have
\begin{equation}\label{basic:lem:eq1}
\Pi(B^c \vert z) \leq e^{\beta} \left(\rho + \sum_{j\geq 1} e^{-r_j}\right).\end{equation}
\end{lemma}
\begin{proof}
Using the lower bound on the normalizing constant provided by the event (\ref{event:e}), for $z\in\e$, we have
\begin{multline*}
\Pi(B^c\vert z) = \frac{\int_{B^c}  \frac{f_\theta(z)}{f(z)} \pi(\rmd \theta)}{\int_{\Theta} \frac{f_\theta(z)}{f(z)} \pi(\rmd \theta)}\leq e^{\beta}\left(\int_{S^c}  \frac{f_\theta(z)}{f(z)} \pi(\rmd \theta) + \int_{S\cap B^c}  \frac{f_\theta(z)}{f(z)} \pi(\rmd \theta)\right)\\
\leq e^{\beta}\left(\rho+ \int_{S\cap B^c}  \frac{f_\theta(z)}{f(z)} \pi(\rmd \theta)\right).
\end{multline*}
Furthermore,  for $z\in\e$, the last integral in the last display satisfies
\[\int_{S\cap B^c} \frac{f_\theta(z)}{f(z)} \pi(\rmd \theta)  \leq  \sum_{j\geq 1} \int_{\Xi_j}\exp\left(\log f_\theta(z) - \log f(z)\right)\pi(\rmd\theta)\leq \sum_{j\geq 1} e^{-r_j}\pi(\Xi_j).\]
Equation (\ref{basic:lem:eq1}) follows by collecting the terms.
\end{proof}
\bigskip
\begin{remark}From the  lemma we are left with the problem of finding $\rho,\beta, \{r_j,\;j\geq1\}$ such that the right hand size of Equation (\ref{basic:lem:eq1}) is small and $\PP(Z\notin \e)$ is small.

 %Typically in applying the lemma, $f$ is taken as the density of the data $Z$. The component $\int_{\Theta}\frac{f_\theta(z)}{f(z)} \pi(\rmd \theta) \geq e^{-\beta}$ of $\e$ is a lower bound on the normalizing constant of $\Pi(\cdot\vert z)$, The second component $\int_{S^c}\frac{f_\theta(z)}{f(z)} \pi(\rmd \theta) \leq \rho$ boils down to a prior contraction requirement, and the last component  is a bound on the log-likelihood empirical process.
\end{remark}

\section{Concluding remarks}
There is a need for a deeper theoretical understanding of deep learning models. We have focused here  on a class of algorithm unrolling models for inverse problems. And we have shown that for convex inverse problems and under a concentration of measure assumption, GDN can recover the inverse map at optimal rate, provided that the unrolling depth is appropriately tuned. These findings are confirmed in our numerical example. Our results also suggest that algorithm unrolling models are prone to overfitting as the unrolling depth $D'$ increases. The theoretical results are obtained as special cases of a more general posterior contraction result for Bayesian deep learning. 

One natural question is whether our analysis extends beyond the concentration of measure  assumption in Assumption \ref{H1:dl}.  
Without the content of H\ref{H1:dl}, a more sensible approach would be to estimate the entire conditional distribution, not just its mean. Several recent works have proposed to extend algorithm unrolling architectures for conditional density estimation in inverse problems (\cite{ardizzone:2018}). Extending our analysis to these conditional density models is an important direction for future research.

Another outstanding challenge not addressed in this work is the computational and memory cost of implementing algorithm unrolling models. Our results suggest that fairly deep (but not too deep) networks are typically needed for optimal performance of GDNs. In practice, the gradient of the loss with respect to $W$ in (\ref{inverse:glut:post}) is typically computed by back-propagation through the entire network of depth $D\times D'$, at a memory cost of order $O(D\times D')$. This often puts severe limitations on the unrolling depth that can be considered  (\cite{putzky:etal:19}). Mitigating this memory cost and easing the implementation of algorithm unrolling architectures (for instance by developing  specialized back-propagation algorithms) is another  important problem for future research.

\bibliography{biblio,biblio_optim,biblio_mcmc,biblio_graph,biblio_RT}
\bibliographystyle{icml2022}

%\bibliography{biblio_graph,biblio_mcmc,biblio_optim,biblio_RT,biblio}

%%%%%%%%%%%%%%%%%%%%%%%%%%%%%%%%%%%%%%%%%%%%%%%%%%%%%%%%%%%%%%%%%%%%%%%%%%%%%%%
%%%%%%%%%%%%%%%%%%%%%%%%%%%%%%%%%%%%%%%%%%%%%%%%%%%%%%%%%%%%%%%%%%%%%%%%%%%%%%%
% APPENDIX
%%%%%%%%%%%%%%%%%%%%%%%%%%%%%%%%%%%%%%%%%%%%%%%%%%%%%%%%%%%%%%%%%%%%%%%%%%%%%%%
%%%%%%%%%%%%%%%%%%%%%%%%%%%%%%%%%%%%%%%%%%%%%%%%%%%%%%%%%%%%%%%%%%%%%%%%%%%%%%%
%\newpage
\appendix

\section{Proofs}\label{sec:proofs}
\subsection{Proof of Theorem \ref{thm:main}}\label{sec:proof:thm:main}
\begin{proof}
We follow the same general steps outlined above in Lemma \ref{basic:lem}. We recall that the dataset is $\D\eqdef ({\bf x}_1,{\bf y}_1),\ldots,({\bf x}_n,{\bf y}_n)$. For $W\in\W$, we define 
\begin{multline}\label{def:ffstar}
f_{ W }(\D) \eqdef \exp\left( -\frac{1}{2\sigma^2}\sum_{i=1}^n \|{\bf x}_i - g_{ W }({\bf y}_i)\|_2^2\right),\;\;\\
\mbox{ and }\;\; f_{\star}(\D) \eqdef \exp\left( -\frac{1}{2\sigma^2}\sum_{i=1}^n \|{\bf x}_i - g({\bf y}_i)\|_2^2\right).\end{multline}
We recall that $\Theta=\W\times \mathcal{S}$. For any measurable set $A\subset\Theta$, we can write the posterior probability $\Pi(A\vert \D)$ as 
\begin{equation}\label{Pi:alternative}
\Pi(A\vert \D) = \frac{\int_A\frac{f_{ \Lambda \odot W }(\D)}{f_\star(\D)}\Pi_0(\rmd  \Lambda ,\rmd  W )}{\int_{\Theta}\frac{f_{ \Lambda \odot W }(\D)}{f_\star(\D)}\Pi_0(\rmd  \Lambda ,\rmd  W )}.\end{equation}

We will repeatedly use the following observation. For $ W \in \W $, we have
\begin{multline}\label{log:ll:eq}
\log\left(\frac{f_{ W }(\D)}{f_\star(\D)}\right) = \frac{1}{2\sigma^2}\sum_{i=1}^n \left(\|{\bf x}_i -g({\bf y}_i)\|_2^2 - \|{\bf x}_i -g_{ W }({\bf y}_i)\|_2^2\right)\\
 = - \frac{n}{2\sigma^2}\|g_{ W } - g\|_n^2 - \frac{1}{\sigma^2}\sum_{i=1}^n\pscal{{\bf x}_i - g({\bf y}_i)}{ g({\bf y}_i) - g_{ W }({\bf y}_i)}.\end{multline}

Given $s_0\geq 1$,  $\beta_0\geq 0$, we set
\[\Theta(s_0,\beta_0)\eqdef\left\{( \Lambda , W )\in \Theta :\;\|\Lambda\|_0 \leq s_0,\;\mbox{ and} \; \|\Lambda\odot W\|_\infty \leq \beta_0\right\},\]
and
\[\W(s_0,\beta_0) \eqdef  \left\{ W \in \W :\; \|W\|_0 \leq s_0,\;\;\ \|W\|_\infty\leq \beta_0\right\}.\]
We set
\[s \eqdef  \frac{1}{\mathsf{u}} + \left(1+ \frac{5}{\mathsf{u}} + \frac{\log(\mathsf{L}_\star\sqrt{n})}{\mathsf{u}\log(q)}\right) s_\star +  \frac{4n\varpi_\star^2}{\sigma^2\mathsf{u}\log(q)},\;\mbox{ and }\;\; \r \eqdef \bar\sigma\sqrt{\frac{s\log(q \mathsf{L}_s)}{n}},\]
and
\[ \bar\alpha\eqdef \mathsf{u} s -1,\]
where $\mathsf{L}_\star \eqdef L(2s_\star^{1/2})$, $\mathsf{L}_s \eqdef L(2s^{1/2}b_s)$, and  $b_{s}\eqdef \sqrt{2\rho_1^{-1}(1+\mathsf{u})(s+1)\log(q)}$, and where $L$ is as in Assumption \ref{H5:dl}. Fix $M\geq 2$. For $j\geq 1$ we also set
\[\W_j(s_0,\beta_0)\eqdef\{W\in\W(s_0,\beta_0):\; j(M\r)<\|g_{ W } - g\|_n \leq (j+1)M\r\}.\]
We shall apply the same idea as in Lemma \ref{basic:lem}. 
Specifically, let
\[B \eqdef\left\{( \Lambda , W )\in\Theta:\; \|g_{ \Lambda \odot  W } - g\|_n\leq M\r \right\},\] 
and consider the $\e$ 
\begin{multline*}
\e= \left\{\D:\; \int_{\Theta}\frac{f_{ \Lambda \odot W }(\D)}{f_{\star}(\D)} \Pi_0(\rmd  \Lambda ,\rmd W ) > \frac{1}{4q^{\bar\alpha}},\;\;\;\int_{\A(s)}\frac{f_{ \Lambda \odot W }(\D)}{f_\star(\D)}\Pi_0(\rmd  \Lambda ,\rmd  W )\leq \frac{1}{q^{\mathsf{u}s}} \right.\\
\left. \mbox{ and }\;\sup_{W\in \W_j(s,b_s)}\; \left[ \log f_W(\D) - \log f_{\star}(\D)\right] \leq -\frac{n(jM\r)^2}{8\sigma^2},\;\mbox{ for all }\; j\geq 1\right\},\end{multline*}
where $\A(s)$ denotes the complement of $\Theta(s,b_s)$. We note if $(\Lambda,W)\in B^c\cap \Theta(s,b_s)$, then $\Lambda \odot W\in \cup_{j\geq 1} \W_j(s,b_s)$.  Let $\check\Pi_0$ be the distribution of $\Lambda\odot W$, when $(\Lambda, W)\sim \Pi_0$. Starting from (\ref{Pi:alternative}),  and following the same argument leading to (\ref{basic:lem:eq1}),  for $\D\in\e$, we have
\begin{eqnarray*}
\Pi(B^c\vert\D ) & \leq &  4q^{\bar\alpha} \int_{B^c}\frac{f_{ \Lambda \odot W }(\D)}{f_\star(\D)}\Pi_0(\rmd  \Lambda ,\rmd  W) \\
& \leq &  4q^{\bar\alpha} \left(\int_{\A(s)}\frac{f_{ \Lambda \odot W }(\D)}{f_\star(\D)}\Pi_0(\rmd  \Lambda ,\rmd  W)  + \int_{B^c\cap \Theta(s,b_s)}\frac{f_{ \Lambda \odot W }(\D)}{f_\star(\D)}\Pi_0(\rmd  \Lambda ,\rmd  W) \right)\\
 & \leq & 4e^{\bar\alpha\log(q)}\left(\frac{1}{q^{\mathsf{u}s}}  + \int_{B^c\cap \Theta(s,b_s)}\frac{f_{ \Lambda \odot W }(\D)}{f_\star(\D)}\Pi_0(\rmd  \Lambda ,\rmd  W) \right)\\
& \leq & 4e^{\bar\alpha\log(q)}\left(\frac{1}{q^{\mathsf{u}s}}  + \sum_{j\geq 1}\int_{\W_j(s,b_s)}\frac{f_{W }(\D)}{f_\star(\D)}\check\Pi_0(\rmd  W) \right)\\
& \leq  & 4e^{\bar\alpha\log(q)}\left(e^{-\mathsf{u}s\log(q)} + \sum_{j\geq 1}e^{-\frac{n(jM\r)^2}{8\sigma^2}}\right) \\
& \leq & 4e^{\bar\alpha\log(q)}\left(e^{-\mathsf{u}s\log(q)} + 2e^{-\frac{n(M\r)^2}{8\sigma^2}}\right).
\end{eqnarray*}
By the definition of $s$ and $\r$ above, we have  $\mathsf{u}s = \bar\alpha +1$, and
\[n(M\r)^2\geq M^2\bar\sigma^2 s\log(q) = M^2\bar\sigma^2 \left(\frac{1 +\bar\alpha}{\mathsf{u}}\right)\log(q) \geq 8\sigma^2(1 +\bar\alpha)\log(q),\]
by taking $M^2\geq 8\mathsf{u}(\sigma^2/\bar\sigma^2)$.  Hence for $\D\in\e$,
\[\Pi(B^c\vert\D )\leq   \frac{12}{q}.\]
This implies that with probability at least $\PP(\D\in \e)$, we have
\[\Pi(B^c\vert\D )\leq \frac{12}{q}.\]
We show in  Lemma \ref{lem:nc:dl} below that
\[\PP\left[\int_{\Theta}\frac{f_{ \Lambda \odot W }(\D)}{f_{\star}(\D)} \Pi_0(\rmd  \Lambda ,\rmd W ) \leq \frac{1}{4q^{\bar\alpha}}\;\vert \;{\bf y}_{1:n}\right]  \leq \frac{4}{q^{s_\star}},
\]
and we show in Lemma  \ref{lem:prior:contr} below that 
\[\PP\left[\int_{\A}\frac{f_{ \Lambda \odot W }(\D)}{f_\star(\D)}\Pi_0(\rmd  \Lambda ,\rmd  W )> \frac{1}{q^{\mathsf{u}s}} \;\vert \;{\bf y}_{1:n}\right]\leq \frac{3}{q^{\mathsf{u}}}.\]
It follows that 
\begin{multline*}
\PP(\D\notin \e\;\vert\; {\bf y}_{1:n}) \leq  \frac{4}{q^{s_\star}} + \frac{3}{q^{\mathsf{u}}} \\
+ \PP\left[\bigcup_{j\geq 1}\left\{\sup_{W\in\W_j(s,b_s)}\; \left[\log f_W(\D) - \log f_\star(\D)\right] > -\frac{n(jM\r)^2}{8\sigma^2}\right\} \; \vert\; {\bf y}_{1:n}\right].\end{multline*}
By Lemma \ref{lem:unif:b}  applied  with $\W_0=\W(s,b_s)$, the rightmost term in the last display is bounded from above by $e^{-c_0n} + 4e^{-n(M\r)^2/(c_0\bar\sigma^2)}$, for some absolute constant $c_0$ provided that the term $\r$ defined above satisfies 
\begin{equation}\label{cond:r}
\frac{288}{\sqrt{n}}\int_{\frac{x^2}{32\bar\varsigma}}^x \sqrt{\log\N\left(\epsilon, \W ^{(x)}(s,b_s),\|\cdot\|_n\right)}\rmd \epsilon  \leq \frac{x^2}{\bar\sigma},\;\;\mbox{ for all }\;\; x\geq \r,\end{equation}
where for $s_0\geq 1$, $x\geq 0$,   $\beta_0\geq 0$ we define
\[\W ^{(x)}(s_0,\beta_0)\eqdef\{ W \in \W(s_0,\beta_0),\;\|g_{ W } - g\|_n \leq x\},\]
and given $\epsilon>0$, and $A\subset\W$, $\N(\epsilon,A,\|\cdot\|_n)$ denotes the cardinality  of a smallest $\epsilon$-cover of $A$ in the pseudo-metric $\|\cdot\|_n$ defined as  $\|W-W'\|_n\eqdef \|g_W - g_{W'}\|_n$.  We therefore reach the conclusion that with probability at least $1 - e^{-c_0n}  - c_1/q$, 
\[\Pi(B^c\vert\D )  \leq \frac{12}{q}.\]
for some absolute constants $c_0,c_1$. It remains to check (\ref{cond:r}). First we use the majoration
\[\int_{\frac{x^2}{32\bar\varsigma}}^x \sqrt{\log\N\left(\epsilon, \W ^{(x)}(s,b_s),\|\cdot\|_n\right)}\rmd \epsilon \leq x \sqrt{\log\N\left(\frac{x^2}{32\bar\varsigma}, \W(s,b_s),\|\cdot\|_n\right)}.\]

We recall that our notation $\|W\|_2$ denotes the Euclidean norm of the vectorized parameter $W$.  For $W\in\W(s,b_s)$, $\|W\|_2 \leq s^{1/2}b_s$. Hence, assumption H\ref{H3:dl}, and the definition of $\mathsf{L}_s = L(s^{1/2}b_s)$ implies that for all $W,W'\in\W(s,b_s)$, we have
\[\|W - W'\|_n = \|g_W-g_{W'}\|_n \leq \mathsf{L}_s\|W -W'\|_{2}.\]
Therefore, we can use the metric entropy of the $s$-sparse ball of $\rset^q$ with radius $s^{1/2}b_s/\mathsf{L}_s$ with respect to the Euclidean norm to get
\[\N(\epsilon, \W (s,b_s),\|\cdot\|_n)  \leq q^s\left(1 + \frac{2 s^{1/2}b_s\mathsf{L}_s}{\epsilon}\right)^s.\]
Hence
\[\frac{288}{\sqrt{n}} \int_{\frac{x^2}{32\bar\varsigma}}^x \sqrt{\log\N(\epsilon, \W ^{(x)}(s,b_s),\|\cdot\|_n)}\rmd \epsilon \leq 288 x\sqrt{\frac{s\log(q)}{n} + \frac{s\log\left(1 + \frac{64\bar\varsigma s^{1/2}b_s\mathsf{L}_s}{x^2}\right)}{n}}.\]
We can insist to search for $x \geq \sqrt{128\bar\varsigma /n}$, and  conclude that the right hand side of the last display is always upper bounded by 
\[288 x\sqrt{\frac{s\log(q)}{n} + \frac{s\log\left(1 + \frac{n s^{1/2} b_s \mathsf{L}_s}{2}\right)}{n}} \leq c_0 x\sqrt{\frac{s\log(q \mathsf{L}_s)}{n}},\]
for some absolute constant $c_0$. The right hand side of the last display is upper bounded by $\frac{x^2}{\bar\sigma}$ for all 
\[x\geq c_0\bar\sigma \sqrt{\frac{s\log(q\mathsf{L}_s)}{n}},\]
hence the theorem, after moving the constant $c_0$ into $M$.
\end{proof}

\medskip

\begin{lemma}\label{lem:prior:contr}
Assume H\ref{H1:dl}, and suppose that $\sigma^2\geq\max_i\sigma_i^2$.  For all integers $s\geq 1$, with $b_s \eqdef \sqrt{2(1+\mathsf{u})(1 +s)\log(q)/\rho_1}$,
we have
\[\PP\left[\int_{\A(s)}\frac{f_{ \Lambda \odot W }(\D)}{f_\star(\D)}\Pi_0(\rmd  \Lambda ,\rmd  W )> \frac{1}{q^{\mathsf{u}s}} \;\vert \;{\bf y}_{1:n}\right]\leq \frac{4}{q^{\mathsf{u}}},\]
where $\A(s)$ denotes the complement of the set $\Theta(s,b_s)$ where
\[\Theta(s,b)\eqdef\left\{( \Lambda , W )\in \Theta :\;\|\Lambda\|_0 \leq s,\;\mbox{ and} \; \|\Lambda\odot W\|_\infty \leq b\right\}.\]
\end{lemma}
\begin{proof}
Since $\A(s)$ is the complement of the set $\Theta(s,b_s)$, we can  write
\[\Pi_0(\A(s)) = \Pi_0(\|\Lambda\|_0 >s) + \sum_{\Lambda:\;\|\Lambda\|_0\leq s} \Pi_0(\Lambda)\times \Pi_0(\|\Lambda\odot W\|_\infty >b_s \vert \Lambda).\]
 If $( \Lambda , W )\sim\Pi_0$, then $ \Lambda $ is an ensemble of iid random variables drawn from the Bernoulli distribution with success probability $(1+ q^{\mathsf{u}+1})^{-1}$. Hence
\begin{multline*}
\Pi_0(\| \Lambda \|_0 > s) \leq \sum_{k> s} {q\choose k}\left(\frac{1}{1 + q^{\mathsf{u}+1}}\right)^{k} \left(\frac{q^{\mathsf{u}+1}}{1 + q^{\mathsf{u}+1}}\right)^{q - k} \\
\leq \sum_{k> s} {q\choose k} \left(\frac{1}{q^{\mathsf{u}+1}}\right)^k \leq 2\left(\frac{1}{q^{\mathsf{u}}}\right)^{s+1},\end{multline*}
where we use ${q\choose k}\leq q^k$, and $q^{\mathsf{u}}\geq 2$. Given $\Lambda_{k}=1$, $W_{k}\sim \mathbf{N}(0,\rho_1^{-1})$. Therefore, $\PP(|W_k|>t)\leq 2e^{-\rho_1 t^2/2}$ for all $t\geq 0$. Hence by union bound, for $\| \Lambda \|_0 \leq s$, we obtain
\[
\Pi_0\left(\|\Lambda\odot W\|_\infty >b_s \; \vert \;  \Lambda \right) \leq
2e^{-\rho_1 b_s^2/2 + \log(s)} \leq \frac{2}{q^{\mathsf{u}(1+s)}}.\]
We conclude that
\begin{equation}\label{control:Pi:A}
\Pi_0(\A(s)) \leq \frac{4}{q^{\mathsf{u}(1+s)}}. 
\end{equation}
 Now, by Markov's inequality, and Fubini's theorem, we have 
\begin{multline*}
\PP\left[\int_{\A(s)}\frac{f_{ \Lambda \odot W }(\D)}{f_\star(\D)}\Pi_0(\rmd  \Lambda ,\rmd  W )> \frac{1}{q^{\mathsf{u}s}} \;\vert \;{\bf y}_{1:n}\right]\\
\leq q^{\mathsf{u}s}\int_{\A(s)}\PE\left[\frac{f_{ \Lambda \odot W }(\D)}{f_\star(\D)} \;\vert \;{\bf y}_{1:n}\right]\Pi_0(\rmd  \Lambda ,\rmd  W ),\end{multline*}
and from (\ref{log:ll:eq}) we have
\[\PE\left[\frac{f_{ \Lambda \odot W }(\D)}{f_\star(\D)} \;\vert \;{\bf y}_{1:n}\right] =  e^{-\frac{n}{2\sigma^2}\|g_{ \Lambda \odot W } - g\|_n^2} \PE\left[e^{-\frac{1}{\sigma^2}\sum_{i=1}^n \pscal{\bm{\xi}_i}{g({\bf y}_i) - g_{ \Lambda \odot W }({\bf y}_i)}}\;\vert \;{\bf y}_{1:n}\right].\]
We have assumed in H\ref{H1:dl} that $\PE(\bm{\xi}_i\vert{\bf y}_i) = 0 $, and $\|\bm{\xi}_i\vert {\bf y}_i\|_{\psi_2}\leq \sigma_i$.  Therefore,
\[\PE\left[e^{-\frac{1}{\sigma^2}\sum_{i=1}^n \pscal{\bm{\xi}_i}{g({\bf y}_i) - g_{ \Lambda \odot W }({\bf y}_i)}}\;\vert \;{\bf y}_{1:n}\right] \leq e^{\frac{1}{\sigma^2}\sum_{i=1}^n\frac{\sigma_i^2d_i^2}{2\sigma^2}},
\]
where $d_i$ is a short for $\|g({\bf y}_i) - g_{ \Lambda \odot W }({\bf y}_i)\|_2$.  We conclude that
\[\PE\left[\frac{f_{ \Lambda \odot W }(\D)}{f_\star(\D)} \;\vert \;{\bf y}_{1:n}\right]  \leq \exp\left(-\frac{1}{2\sigma^2}\sum_{i=1}^n\left[\left(1- \frac{\sigma^2_i}{\sigma^2}\right)d_i^2\right]\right).\]
And we easily check that for $\sigma^2 \geq \bar\sigma^2$, the right hand size of the last display is bounded from above by $1$.  We conclude that 
\[  \PP\left[\int_{\A(s)}\frac{f_{ \Lambda \odot W }(\D)}{f_\star(\D)}\Pi_0(\rmd  \Lambda ,\rmd  W )> \frac{1}{q^{\mathsf{u}s}} \;\vert \;{\bf y}_{1:n}\right]\leq q^{\mathsf{u}s} \Pi_0(\A(s)) \leq \frac{4}{q^{\mathsf{u}}}.
\]

\end{proof}

\medskip

The next result lower bounds the normalizing constant of $\Pi(\cdot\vert\D)$. 

\begin{lemma} \label{lem:nc:dl}
Under the assumption of Theorem \ref{thm:main} it holds,
\[\PP\left[\int_{\Theta}\frac{f_{ \Lambda \odot W }(\D)}{f_{\star}(\D)} \Pi_0(\rmd  \Lambda ,\rmd W ) \leq \frac{1}{4q^{\bar\alpha}}\;\vert \;{\bf y}_{1:n}\right]  \leq \frac{4}{q^{s_\star}},
\]
where
\[\bar\alpha\eqdef   \left(\mathsf{u} + 5 + \frac{\log(\mathsf{L}_\star\sqrt{n})}{\log(q)}\right)s_\star + \frac{4n\varpi_\star^2}{\sigma^2\log(q)}.\]
\end{lemma}
\begin{proof}
By the assumption of Theorem \ref{thm:main}, we can find $W_\star$ with $\|W_\star\|_0\leq s_\star$, $\|W_\star\|_\infty\leq 1$, such that $\|g_{W_\star} - g\|_n\leq \varpi_\star$. Let  $ \Lambda _\star$ denote the  sparsity support of $W_\star$. With $\mathsf{L}_\star = L(2s_\star^{1/2})$, we  set
\[\eta \eqdef 1\wedge \frac{\sigma}{\mathsf{L}_\star}\sqrt{\frac{\log(q)}{n}},\;\;\;\mbox{ and }\;\;\; \N(\eta)  \eqdef \left\{  W\in\W:\;\;  \;\; \; \;  \|W\odot\Lambda_\star-W_{\star}\|_\infty \leq \eta\right\}.\]
We see that $\|W_{\star}\|_{2} \leq s_\star^{1/2}$, and for $W\in\N(\eta)$, 
\[
\|W\odot\Lambda_\star\|_{2} \leq \|W_\star\|_{2}  + \|W_{\star} - W\odot\Lambda_\star\|_{2} \leq s_\star^{1/2}  + s_\star^{1/2} \eta \leq 2 s_\star^{1/2}.\]
Therefore, by H\ref{H5:dl} applied with $\eta=2s_\star^{1/2}$, for all $W\in\mathcal{N}(\eta)$, we have
\[
\max_{1\leq i\leq n}\; \|g_{ \Lambda_\star\odot W }({\bf y}_i) - g_{ W _\star}({\bf y}_i)\|_2 \leq \mathsf{L}_\star \|\Lambda_\star\odot W-W_\star\|_2 \leq \mathsf{L}_\star \sqrt{s_\star}\eta \leq \sigma \sqrt{\frac{s_\star \log(q)}{n}}.\]
Hence
\begin{equation} \label{control:N}
\max_{1\leq i\leq n}\;\sup_{ W \in\N(\eta)} \|g_{ \Lambda_\star\odot W }({\bf y}_i) - g_{ W _\star}({\bf y}_i)\|_2 \leq \sigma \sqrt{\frac{s_\star  \log(q)}{n}}.
\end{equation}
Switching the sign and taking the conditional expectation in (\ref{log:ll:eq}) using $\PE({\bf x}_i \vert {\bf y}_i) = g({\bf y}_i)$, yields
\[
\PE\left[\log\left(\frac{f_\star(\D)}{f_{ \Lambda_\star\odot W }(\D)}\right)\;\vert\;{\bf y}_{1:n}\right]  =  \frac{1}{2\sigma^2}\sum_{i=1}^n \|g_{ \Lambda_\star\odot W }({\bf y}_i) - g({\bf y}_i)\|_2^2,\]
and we conclude using  (\ref{control:N}) and the definition of $\varpi_\star$ and $\W_\star$ in Theorem \ref{thm:main} that
\begin{multline*}
\sup_{ W \in\N(\eta)} \;\PE\left[\log\left(\frac{f_\star(\D)}{f_{ \Lambda_\star\odot W }(\D)}\right)\;\vert\;{\bf y}_{1:n}\right] \\
\leq \frac{n \varpi_{\star}^2}{\sigma^2}+ \frac{1}{\sigma^2}\sum_{i=1}^n \sup_{ W \in\N(\eta)} \;\|g_{ \Lambda_\star\odot W }({\bf y}_i) - g_{W_\star}({\bf y}_i)\|_2^2  \\
\leq  \frac{n \varpi_{\star}^2}{\sigma^2} + s_\star \log(q).\end{multline*}
Going back to (\ref{log:ll:eq}), we have
\begin{equation}\label{eq:centered:ll}
\log\left(\frac{f_\star(\D)}{f_{ \Lambda_\star\odot W }(\D)}\right) - \PE\left[\log\left(\frac{f_\star(\D)}{f_{ \Lambda_\star\odot W }(\D)}\right)\;\vert\;{\bf y}_{1:n}\right] 
 = \frac{1}{\sigma^2}\sum_{i=1}^n\pscal{\bm{\xi}_i}{g({\bf y}_i) - g_{ \Lambda_\star\odot W }({\bf y}_i) }.
 \end{equation}
 We use the notation $\|Z\|_{\psi_2}$ to denote the sub-Gaussian norm of the conditional law of the random variable $Z$ given ${\bf y}_{1:n}$. By conditional independence of the error terms $\bm{\xi}_i$, for all $ W \in \N(\eta)$, we have
 \begin{multline*}
\left\| \log\left(\frac{f_\star(\D)}{f_{ \Lambda_\star\odot W }(\D)}\right) - \PE\left[\log\left(\frac{f_\star(\D)}{f_{ \Lambda_\star\odot W }(\D)}\right)\;\vert\;{\bf y}_{1:n}\right]\right\|_{\psi_2}^2\\
\leq \frac{1}{\sigma^4}\sum_{i=1}^n \left\|\pscal{\bm{\xi}_i}{g({\bf y}_i) - g_{ \Lambda_\star\odot W }({\bf y}_i) }\right\|_{\psi_2}^2 \\
 =\frac{1}{\sigma^4}\sum_{i=1}^n \sigma_i^2 \|g({\bf y}_i) - g_{ \Lambda_\star\odot W }({\bf y}_i)\|_2^2\leq \frac{2n\varpi_{\star}^2}{\sigma^2} + 2s_\star \log(q).
 \end{multline*}
In the sequel,  we set 
\[a\eqdef 2\left(\frac{n\varpi_{\star}^2}{\sigma^2} + s_\star  \log(q)\right).\] 
Then by Hoeffding's inequality, for all $ W \in \N(\eta)$, we have
\[
\PP\left[\left|\log\left(\frac{f_\star(\D)}{f_{ \Lambda_\star\odot W }(\D)}\right) - \PE\left[\log\left(\frac{f_\star(\D)}{f_{ \Lambda_\star\odot W }(\D)}\right)\;\vert\;{\bf y}_{1:n}\right]\right| > a\;\vert\;{\bf y}_{1:n}\right]\leq 2 e^{-a/2}\leq \frac{2}{q^{s_\star}}.\]
We can rewrite this statement in the following equivalent form. For $ W \in \W $, define
\[\e_{ W }\eqdef\left\{\D:\; \left|\log\left(\frac{f_\star(\D)}{f_{ W }(\D)}\right) - \PE\left[\log\left(\frac{f_\star(\D)}{f_{ W }(\D)}\right)\;\vert\;{\bf y}_{1:n}\right]\right|\leq a\right\}.\]
We have
\begin{equation}\label{proof:nc:bound:eq:evt}
\sup_{ W \in\mathcal{N}(\eta)} \PP\left(\D\notin \e_{ \Lambda_\star\odot W }\;\vert\;{\bf y}_{1:n}\right) \leq  \frac{2}{q^{s_\star}}.
\end{equation}
Using these observations, we have
\begin{multline*}
\int_{\Theta }\frac{f_{ \Lambda \odot W }(\D)}{f_{\star}(\D)} \Pi_0(\rmd  \Lambda ,\rmd W ) \geq 
\int_{ \Lambda _\star\times \N(\eta)} e^{-\PE\left[\log\left(\frac{f_{\star}(\D)}{f_{ \Lambda \odot W }(\D)}\right)\;\vert\;{\bf y}_{1:n}\right]}\\
\times \exp\left(-\left[\log\left(\frac{f_{\star}(\D)}{f_{ \Lambda \odot W }(\D)}\right) - \PE\left[\log\left(\frac{f_{\star}(\D)}{f_{ \Lambda \odot W }(\D)}\right)\;\vert\;{\bf y}_{1:n}\right]\right]\right)\textbf{1}_{\e_{ \Lambda \odot W }}(\D) \Pi_0(\rmd  \Lambda ,\rmd W ) \\
 \geq  e^{-2a} \int_{ \Lambda _\star\times \mathcal{N}(\eta)}\textbf{1}_{\e_{ \Lambda \odot W }}(\D) \Pi_0(\rmd  \Lambda ,\rmd W ) \\
 =  e^{-2a}\left(\Pi_0( \Lambda _\star\times \N(\eta)) - \int_{ \Lambda _\star\times \N(\eta)} \textbf{1}_{\e_{ \Lambda \odot W }^c}(\D)\Pi_0(\rmd  \Lambda ,\rmd W )\right).
\end{multline*}
Therefore, by Chebyshev's inequality, 
\begin{multline*}
\PP\left[\int_{\Theta }\frac{f_{ \Lambda \odot W }(\D)}{f_{\star}(\D)} \Pi_0(\rmd  \Lambda ,\rmd W )  \leq e^{-2a}\Pi_0( \Lambda _\star\times \N(\eta))/2 \;\vert \; {\bf y}_{1:n}\right]\\
 \leq \PP\left[  \int_{ \Lambda _\star\times \N(\eta)} \textbf{1}_{\e_{ \Lambda \odot W }^c}(\D)\Pi_0(\rmd  \Lambda ,\rmd W )  \geq \frac{1}{2}\Pi_0( \Lambda _\star\times \N(\eta)) \;\vert \;{\bf y}_{1:n}\right]\\
\leq \frac{2}{\Pi_0( \Lambda _\star\times \N(\eta))}  \int_{ \Lambda _\star\times \N(\eta)}\PP\left(\D \notin\e_{ \Lambda_\star\odot W }\;\vert \;{\bf y}_{1:n}\right)\Pi_0(\rmd  \Lambda ,\rmd W ) \\
\leq 2\sup_{W\in\mathcal{N}(\eta)} \PP_\star\left(\D\notin \e_{ \Lambda_\star\odot W } \;\vert \;{\bf y}_{1:n}\right) \leq \frac{4}{q^{s_\star}},
\end{multline*}
using (\ref{proof:nc:bound:eq:evt}). To conclude the proof it remains only to lower bound $\Pi_0( \Lambda _\star \times \mathcal{N}(\eta))$.
Since $\log(1-x)\geq -2x$ for all $0\leq x\leq 1/2$, for $q^{\mathsf{u}} \geq 2/\log(2)$, we have
\begin{multline*}
\Pi_0( \Lambda _\star) = \left(\frac{1}{1+ q^{\mathsf{u}+1}}\right)^{\| \Lambda _{\star}\|_0} \left(1 - \frac{1}{1+ q^{\mathsf{u}+1}}\right)^{q - \| \Lambda _{\star}\|_0} \\
= \left(\frac{1}{q^{\mathsf{u}+1}}\right)^{\| \Lambda _{\star}\|_0} \exp\left(q\log\left(1 - \frac{1}{1+ q^{\mathsf{u}+1}}\right)\right)\\
\geq \left(\frac{1}{q^{\mathsf{u}+1}}\right)^{\| \Lambda _{\star}\|_0} \exp\left(-\frac{2q}{1+ q^{\mathsf{u}+1}}\right) \geq \frac{1}{2} \left(\frac{1}{q^{\mathsf{u}+1}}\right)^{\| \Lambda _{\star}\|_0} \geq \frac{1}{2} \left(\frac{1}{q^{\mathsf{u}+1}}\right)^{s_\star}.
\end{multline*}
If $U\sim \textbf{N}(0,\rho_1)$, then $P(|U-a|\leq t) \geq P(|a| \leq U\leq |a| +t)$ for all $t\geq 0$. We use this inequality to deduce that
\begin{multline*}
\Pi_0(\mathcal{N}(\eta) \;\vert\;  \Lambda _\star) \geq \left(\Phi(\sqrt{\rho_1}(1+\eta)) - \Phi(\sqrt{\rho_1})\right)^{\| \Lambda _{\star}\|_0} \geq \left(c_0\sqrt{\rho_1}\eta\right)^{s_\star}\\
\geq \left(\frac{c_0\sigma}{\mathsf{L}_\star}\sqrt{\frac{\rho_1\log(q)}{n}}\right)^{s_\star}\geq \left(\frac{1}{\mathsf{L}_\star\sqrt{n}}\right)^{s_\star},
\end{multline*}
for some absolute constant $c_0$ ($c_0$ can be taken as $e^{-2}/\sqrt{2\pi}$, since $\rho_1=1$), where $\Phi$ is the cdf of the standard normal distribution. The last inequality in the last display uses the assumption that $n\geq \sigma^2\log(p)$, and $c_0^2\sigma^2\log(q)\geq 1$. We conclude that
\begin{multline*}
e^{-2a}\Pi_0( \Lambda _\star\times \N(\eta)) \\
\geq\frac{1}{2}\exp\left(-\frac{4n\varpi_\star^2}{\sigma^2} - 4s_\star\log(q) - (\mathsf{u}+1)s_\star\log(q) - s_\star\log\left(\mathsf{L}_\star\sqrt{n}\right)\right),\\
\geq \frac{1}{2}\exp\left(-\frac{4n\varpi_\star^2}{\sigma^2} - (\mathsf{u}+5)s_\star\log(q)  - s_\star\log(\mathsf{L}_\star\sqrt{n})\right).
\end{multline*}
Hence the result.
\end{proof}

\begin{lemma}\label{lem:unif:b}
Suppose that the dataset $\D$ is generated as in H\ref{H1:dl}, and consider the nonparametric regression (\ref{au:model})  for some function class $\{g_W,\;W\in\W\subseteq\rset^q\}$. Let $ \W _0$ be some subset of $ \W$. Suppose that we can find $\r>0$ such that for all $x\geq \r$, it holds
\begin{equation}\label{def:rate:eq}
\frac{288}{\sqrt{n}}\int_{\frac{x^2}{16\bar\varsigma}}^x \sqrt{\log\N(\epsilon, \W ^{(x)},\|\cdot\|_n)}\rmd \epsilon \leq \frac{x^2}{\bar\sigma},\end{equation}
where $ \W ^{(x)}\eqdef\{ W \in \W _0,\;\|g_{ W } - g\|_n \leq x\}$. Let $f_W$ and $f_\star$ be as defined in (\ref{def:ffstar}). Then there exists an absolute constant $c_0$ such that for all $M\geq 1$, such that $n(M\r)^2 \geq c_0\bar\sigma^2$, it holds
\[\PP\left[\bigcup_{j\geq 1}\left\{\sup_{ W \in\widetilde{\W}^{(j)}}\;\;\log\left(\frac{f_{ W }(\D)}{f_{\star}(\D)}\right)  > -\frac{n(jM\r)^2}{8\sigma^2}\right\} \;\vert \; {\bf y}_{1:n}\right]\leq e^{-c_0n} +4e^{-\frac{n M^2\r^2}{c_0\bar\sigma^2}},\]
where $\widetilde{\W}^{(j)}\eqdef \{ W \in \W _0:\; jM\r<\|g_{ W } - g\|_n \leq (j+1)M\r\}$.
\end{lemma}
\begin{proof}
We proceed as in Lemma 3.2 of \cite{vandeGeer:2000}. Throughout the proof, all expectations and probability are conditional given ${\bf y}_{1:n}$. However to ease notation we omit the conditioning. With $M$ and $\r$ as in the statement, and for each integer $j$, we set $\r_j = M\r j$.  We recall the definition of the error terms $\bm{\xi}_i \eqdef {\bf x}_i  - g({\bf y}_i)$, and we define 
\[Z_n(g_{ W }) \eqdef \frac{1}{n\sigma^2}\sum_{i=1}^n\pscal{\bm{\xi}_i}{ g({\bf y}_i) - g_{ W }({\bf y}_i)},\;\;\; W \in \W .\]
Using (\ref{log:ll:eq})  we can re-express the log-likelihood ratio as
\begin{equation}\label{llr:eq1}
\log\left(\frac{f_{ W }(\D)}{f_{\star}(\D)}\right) = - \frac{n}{2\sigma^2}\|g_{ W } - g\|_n^2  - n Z_n(g_{ W }).
\end{equation}

 Let $\varsigma_i$ denote the sub-Gaussian norm of $\|\bm{\xi}_i\|_2$,
 \[\bar\varsigma \eqdef \max_{1\leq i\leq n} \varsigma_i.\] 
 The sub-Gaussian assumption on $\bm{\xi}_i$ implies that $\varsigma_i<\infty$ (see e.g. Theorem 6.3.2 of \cite{vershynin:18}), and that $\|\bm{\xi}_i\|_2^2$ is sub-exponential, with sub-exponential norm $\varsigma_i^2$. We note also that $\PE(\|\bm{\xi}_i\|_2^2) \leq 2\varsigma_i^2$. Therefore, by Bernstein inequality (see e.g. Theorem 2.8.1 of ~\cite{vershynin:18}),
\[
\PP\left(\frac{1}{n}\sum_{i=1}^n \|\bm{\xi}_i\|_2^2 >3\bar\varsigma^2\right)  \leq  \PP\left(\sum_{i=1}^n \|\bm{\xi}_i\|_2^2-\PE(\|\bm{\xi}_i\|_2^2) >n\bar\varsigma^2\right) 
\leq e^{-c_0n},\]
for some absolute constant $c_0$. To make use of this bound, we define 
\[\F_0 \eqdef \left\{\D:\; \sum_{i=1}^n \|\bm{\xi}_i\|_2^2 \leq 3n\bar\varsigma^2\right\}.\]  
Therefore,
\[
\PP\left[\bigcup_{j\geq 1}\left\{\sup_{ W \in\widetilde{\W}^{(j)}}\;\;\log\left(\frac{f_{ W }(\D)}{f_{\star}(\D)}\right)  > -\frac{n\r_j^2}{8\sigma^2}\right\} \right] \leq e^{-c_0n} +\sum_{j\geq 1}\PP\left[\F_j\right],
\]
where 
\[\F_j \eqdef \F_0 \bigcap \left\{\sup_{ W \in\widetilde{\W}^{(j)}}\;\;\log\left(\frac{f_{ W }(\D)}{f_{\star}(\D)}\right)  > -\frac{n\r_j^2}{8\sigma^2}\right\} .\]
For each $j\geq 1$, we set 
\[ \W ^{(j)}\eqdef \{ W \in \W _0:\; \|g_{ W } - g\|_n \leq \r_{j+1}\}, \]
and each $\iota=1,\ldots$, let $\Cset_j^{(\iota)}\eqdef \{g^{(\iota)}_{j,1},\ldots,g^{(\iota)}_{j,N_{j,\iota}}\}$ be a $(\r_{j+1}2^{-\iota})$-covering of $ \W ^{(j)}$. For $\iota=0$, we set $\Cset_j^{(0)} = \{g\}$. The definition implies that for any $ W \in \W ^{(j)}$, we can find $g_{j, W }^{(\iota)}\in\Cset_j^{(\iota)}$ such that $\|g_{ W } -g_{j, W }^{(\iota)}\|_n\leq \r_{j+1}2^{-\iota}$. Let $\ell_j\geq 0$, be the smallest integer such that 
\[\frac{\r_{j+1}}{2^{\ell_j}} \leq \frac{\r_j^2}{16\bar\varsigma}.\]
We consider separately the cases $\ell_j=0$ and $\ell_j>0$. 
\medskip

\paragraph{\texttt{Suppose $\ell_j=0$}.}\;\; In that case for any $ W \in \W ^{(j)}$, $\|g_{ W } -g\|_n\leq \r_{j+1} \leq \r_j^2/(16\bar\varsigma)$. Therefore, on the event $\F_0$, we have
\begin{multline*}
\sup_{ W \in \W _j}\left|Z_n(g_{ W })\right| \leq \frac{1}{n\sigma^2} \sum_{i=1}^n \|\bm{\xi}_i\|_2 \|g_{ W }({\bf y}_i) - g({\bf y}_i) \|_2 \leq \frac{\sqrt{3\bar\varsigma^2}}{\sigma^2}\|g_{ W } -g\|_n \\
\leq  \frac{\sqrt{3\bar\varsigma^2}}{\sigma^2} \frac{\r_j^2}{16\bar\varsigma} \leq \frac{\r_j^2}{8\sigma^2}.
\end{multline*}
Taking this conclusion to (\ref{llr:eq1}) implies that on $\F_0$,
\[\sup_{ W \in\widetilde{\W}^{(j)}}\; \log\left(\frac{f_{ W }(\D)}{f_{\star}(\D)}\right)  \leq -\frac{n\r_j^2}{2\sigma^2} + n\sup_{{ W }\in \W _j}\left|Z_n(g_{ W })\right|\leq -\frac{n\r_j^2}{4\sigma^2}.\]
Hence, when $\ell_j=0$, $\PP(\F_j)=0$.
\medskip

\paragraph{\texttt{Suppose $\ell_j>0$}.}\;\; Similarly, on the event $\F_0$,  we have
\begin{multline*}
\left|Z_n(g_{ W }) - Z_n(g_{j, W }^{(\ell_j)})\right| \leq \frac{1}{n\sigma^2} \sum_{i=1}^n \|\bm{\xi}_i\|_2 \|g_{j, W }^{(\ell_j)}({\bf y}_i) - g_{ W }({\bf y}_i)\|_2 \\
\leq \frac{\sqrt{3\bar\varsigma^2}}{\sigma^2}\|g_{j, W }^{(\ell_j)} - g_{ W }\|_n \leq \frac{\sqrt{3\bar\varsigma^2}}{\sigma^2} \frac{\r_{j+1}}{2^{\ell_j}} \leq  \frac{\sqrt{3\bar\varsigma^2}}{\sigma^2} \frac{\r_j^2}{16\bar\varsigma} \leq \frac{\r_j^2}{8\sigma^2}.
\end{multline*}
This implies that on $\F_0$,
\begin{multline*}
\sup_{ W \in\widetilde{\W}^{(j)}}\; \log\left(\frac{f_{ W }(\D)}{f_{\star}(\D)}\right)  \leq -\frac{n\r_j^2}{2\sigma^2} + n\sup_{ W \in \W ^{(j)}}\left|Z_n(g_{ W }) - Z_n(g_{j, W }^{(\ell_j)})\right| + n\sup_{ W \in \W ^{(j)}}\left|Z_n(g_{j, W }^{(\ell_j)})\right|\\
\leq -\frac{3n\r_j^2}{8\sigma^2} + n\sup_{ W \in \W ^{(j)}}\left|Z_n(g_{j, W }^{(\ell_j)})\right|.\end{multline*}
Hence
\[\PP(\F_j) \leq \PP\left[\sup_{ W \in \W ^{(j)}}\left|Z_n(g_{j, W }^{(\ell_j)})\right| > \frac{\r_j^2}{4\sigma^2}\right].\]
To bound this latter term  we introduce
\begin{multline*}
\delta_j \eqdef \int_{\frac{\r_{j+1}^2}{64\bar\varsigma}}^{\r_{j+1}}\sqrt{\log\N\left(\epsilon, \W ^{(j)},\|\cdot\|_n\right)} \rmd \epsilon,\;\;\\
\mbox{ and }\;\;
\eta_{j,\iota} \eqdef \max\left(\frac{1}{6}\frac{\iota^{1/2}}{2^\iota},\frac{\sqrt{\log N_{j,\iota}}}{4\delta_j} \frac{\r_{j+1}}{2^\iota}\right),\;\;\iota=1,\ldots,\ell_j.\end{multline*}
and  we  write $g_{j, W }^{(\ell_j)}$ as a telescoping sum
\[g_{j, W }^{(\ell_j)} - g = \sum_{\iota=1}^{\ell_j} g_{j, W }^{(\iota)} - g_{j, W }^{(\iota -1)},\]
so that
\[\sup_{ W \in \W ^{(j)}}\left|Z_n(g_{j, W }^{(\ell_j)})\right| \leq  \sum_{\iota=1}^{\ell_j} \;\sup_{ W \in \W ^{(j)}}\left|\frac{1}{n\sigma^2}\sum_{i=1}^n \pscal{\bm{\xi}_i}{g_{j, W }^{(\iota-1)}({\bf y}_i) - g_{j, W }^{(\iota)}({\bf y}_i)}\right|.\]
We show below that the sequence $\{\eta_{j,\iota},\;\iota=1,\ldots,\ell_j\}$ introduced above satisfies
\begin{equation}\label{good:eta:seq}
\sum_{\iota =1}^{\ell_j} \eta_{j,\iota} \leq 1.\end{equation}
Due to (\ref{good:eta:seq}), we can use the sequence $\{\eta_{j,\iota},\;\iota=1,\ldots,\ell_j\}$ to say that
\begin{multline*}
\PP\left[\sup_{ W \in \W ^{(j)}}\left|Z_n(g_{j, W }^{(\ell_j)})\right| >\frac{\r_j^2}{4\sigma^2}\right]\\
\leq \sum_{\iota=1}^{\ell_j} \PP\left[\sup_{ W \in \W ^{(j)}}\left|\frac{1}{n\sigma^2}\sum_{i=1}^n \pscal{\bm{\xi}_i}{g_{j, W }^{(\iota-1)}({\bf y}_i) - g_{j, W }^{(\iota)}({\bf y}_i)}\right| > \frac{\eta_{j,\iota}\r_j^2}{4\sigma^2}\right].\end{multline*}
The supremum on the right-hand side of the last display is in fact a max over a finite set of cardinality at most $N_{j,\iota-1}\times N_{j,\iota}\leq N_{j,\iota}^2$, and for $ W \in \W ^{(j)}$,
\begin{multline*}
\frac{1}{n^2\sigma^4}\sum_{i=1}^n \sigma_i^2\|g_{j,W}^{(\iota-1)}({\bf y}_i) - g_{j,W}^{(\iota)}({\bf y}_i)\|_2^2\\
\leq \frac{2\max_i\sigma_i^2}{n^2\sigma^4}\left(n\|g_{W} - g_{j,W}^{(\iota)}\|_n^2 + n\|g_{W} - g_{j,W}^{(\iota-1)}\|_n^2\right)\leq \frac{10}{n}\max_i\left(\frac{\sigma_i^2}{\sigma^4}\right) \frac{\r_{j+1}^2}{2^{2\iota}}.
\end{multline*}
Therefore by Hoefdding's inequality,
\[\PP\left[\sup_{ W \in \W ^{(j)}}\left|Z_n(g_{j, W }^{(\ell_j)})\right| >\frac{\r_j^2}{4\sigma^2}\right]\leq  \sum_{\iota=1}^{\ell_j} \exp\left(2\log N_{j,\iota} -\frac{n2^{2\iota}\eta_{j,\iota}^2\r_j^4}{(20\times 16)\bar\sigma^2 \r_{j+1}^2}\right).\]
By construction, 
\[\frac{2^{2\iota}\eta_{j,\iota}^2}{\r_{j+1}^2} \geq \frac{\log N_{j\iota}}{16\delta_j^2},\]
which gives
\[\frac{n2^{2\iota}\eta_{j,\iota}^2\r_j^4}{(20\times 16)\bar\sigma^2 \r_{j+1}^2}  \geq \frac{n\r_j^4}{(20\times 32^2)\bar\sigma^2\delta_j^2}\times (4\log N_{j\iota})\geq 4\log N_{j\iota},\]
using (\ref{def:rate:eq}). Therefore 
\[2\log N_{j,\iota} -\frac{n2^{2\iota}\eta_{j,\iota}^2\r_j^4}{(20\times 16)\bar\sigma^2 \r_{j+1}^2} \leq -\frac{n2^{2\iota}\eta_{j,\iota}^2\r_j^4}{(20\times 32)\bar\sigma^2 \r_{j+1}^2} \leq - \frac{n\r_j^2  \iota}{(80\times 36\times 32)\bar\sigma^2},\]
where the last inequality uses the fact that $2^{2\iota}\eta_{j,\iota}^2\geq \iota/36$. It follows that
\[\PP\left[\sup_{ W \in \W ^{(j)}}\left|Z_n(g_{j, W }^{(\ell_j)})\right| >\frac{\r_j^2}{4\sigma^2}\right] \leq\sum_{\iota =1}^{\ell_j}\exp\left(-\frac{n\r_j^2  \iota}{c_0\bar\sigma^2}\right) \leq 2\exp\left(-\frac{n\r_j^2}{c_0\bar\sigma^2}\right),\]
since $n\r_j^2\geq c_0 \bar\sigma^2\log(2)$, for some constant $c_0$ that can be taken as $c_0 = 80\times 36\times 32$.  In conclusion,
\begin{multline*}
\PP\left[\bigcup_{j\geq 1}\left\{\sup_{ W \in\widetilde{\W}^{(j)}}\;\;\log\left(\frac{f_{ W }(\D)}{f_{\star}(\D)}\right)  > -\frac{n\r_j^2}{8\sigma^2}\right\} \right] \\
\leq e^{-c_0n} +2\sum_{j\geq 1}\exp\left(-\frac{n\r_j^2}{c_0\bar\sigma^2}\right)\leq e^{-c_0n} +4\exp\left(-\frac{n M\r^2}{c_0\bar\sigma^2}\right).
\end{multline*}

To check (\ref{good:eta:seq}),  we note 
\[\sum_{\iota=1}^{\ell_j} \eta_{j,\iota} \leq \frac{1}{6}\sum_{\iota=1}^{\ell_j}  \frac{\iota^{1/2}}{2^\iota} + \frac{1}{4\delta_j} \sum_{\iota=1}^{\ell_j}  \frac{\r_{j+1}}{2^\iota}\sqrt{\log N_{j,\iota}}.\]
The function $h(x) = x^{1/2}2^{-x} = x^{\alpha-1}e^{-\beta x}$, with $\alpha=3/2$, $\beta=\log(2)$ is decreasing for $x\geq 1$. Hence 
\[\sum_{\iota\geq 1} \frac{\iota^{1/2}}{2^\iota} = \frac{1}{2} + \sum_{\iota\geq 2} h(\iota) \leq \frac{1}{2} + \sum_{k\geq 2}\int_{k-1}^{k} h(x)\rmd x\leq \frac{1}{2} + \int_1^\infty x^{\alpha-1}e^{\beta x}\rmd x \leq 3. \]
Whereas, 
\begin{multline*}
\sum_{\iota=1}^{\ell_j} \frac{\r_{j+1}}{2^\iota} \sqrt{\log N_{j,\iota}} = \sum_{\iota=1}^{\ell_j}2\int_{\frac{\r_{j+1}}{2^{\iota+1}}}^{\frac{\r_{j+1}}{2^\iota}} \sqrt{\log\N\left(\frac{\r_{j+1}}{2^\iota}, \W ^{(j)},\|\cdot\|_n\right)} \rmd \epsilon\\
\leq  2\int_{\frac{\r_{j+1}}{2^{\ell_j +1}}}^{\frac{\r_{j+1}}{2}}\sqrt{\log\N\left(\epsilon, \W ^{(j)},\|\cdot\|_n\right)} \rmd \epsilon \\
\leq 2\int_{\frac{\r_{j+1}^2}{64\bar\varsigma}}^{\r_{j+1}}\sqrt{\log\N\left(\epsilon, \W ^{(j)},\|\cdot\|_n\right)} \rmd \epsilon =2 \delta_j.
\end{multline*}

\end{proof}

\subsection{Proof of Theorem \ref{thm:2}}\label{sec:proof:thm:2}
We apply  Theorem \ref{thm:main}. The argument has two main steps. First, we  show that the function $g$ can be well approximated by elements of the function class $\{g_W,\;W\in\W\}$ constructed in (\ref{fun:dnn}), and secondly we show that the functions $g_W$ are locally Lipschitz and we estimate the local Lipschitz constant. Both steps rely on  a well-known telescoping argument that we outline first (see e.g. Proposition 6 of \cite{taheri:etal:21}). Given two functions $f = f_K\circ \cdots\circ f_1$, and $g = g_K\circ \cdots\circ g_1$, we write $f-g$ as a telescoping sum
\begin{equation}
\label{telesc:id}
f({\bf x}) - g({\bf x}) = \sum_{j=1}^K f_K\circ \cdots\circ f_j\left(g_{j-1}\circ\cdots \circ g_1({\bf x})\right) - f_K\circ \cdots\circ f_{j+1}\circ g_j\left(g_{j-1}\circ\cdots \circ g_1({\bf x})\right),\end{equation}
 with the convention that for $j=1$, $g_{j-1}\circ\cdots \circ g_1$ is the identity map, and for $j=K$, $f_K\circ \cdots\circ f_{j+1}$ is the identity map. A bound on $\|f({\bf x}) - g({\bf x})\|$ can then be derived using the Lipschitz and boundedness properties of the functions $f_j,g_j$. 
 
Specifically, define $H^{(0)}_W( {\bf x}) \eqdef {\bf x}$, and for $1\leq \ell\leq D$, define $H^{(\ell)}_W( {\bf x}) \eqdef \Psi_{W_\ell}^{(\ell)}(H_W^{(\ell-1)}( {\bf x}))$, so that $H_W({\bf x}) = H^{(D)}_W( {\bf x})$. We recall that 
 \[\Psi_M^{(\ell)}({\bf x}) = \mathsf{a}_\ell(M{\bf x}),\]
where the activation functions $\mathsf{a}_\ell:\;\rset^{p_\ell}\to\rset^{p_\ell}$ are Lipschitz with constant $1$. Then for $1\leq \ell\leq D$, and all $W,W'\in\W$, ${\bf x}_1, {\bf x}_2\in\rset^{d_x}$, by the Lipschitz property of the activation functions $\mathsf{a}_\ell$, we have
\begin{multline*}
\|H^{(\ell)}_W( {\bf x}_1) - H^{(\ell)}_W( {\bf x}_2)\|_2  \leq \|W_\ell H_W^{(\ell-1)}({\bf x}_1) - W_\ell H_W^{(\ell-1)}({\bf x}_2)\|_2\\
\leq \|W_\ell\|_{\textsf{op}} \|H_W^{(\ell-1)}({\bf x}_1) - H_W^{(\ell-1)}({\bf x}_2)\|_2,\end{multline*}
where $\|\cdot\|_{\textsf{op}}$ denotes the operator norm. Iterating this yields,
\begin{equation}\label{lip:h:1}
\|H^{(\ell)}_W( {\bf x}_1) - H^{(\ell)}_W( {\bf x}_2)\|_2 \leq \prod_{j=1}^\ell \|W_j\|_{\textsf{op}}\|{\bf x}_1 - {\bf x}_2\|_{2}.\end{equation}
 Similarly, for any $1\leq \ell\leq D$, (\ref{telesc:id}) gives
\begin{multline*}
H_W^{(\ell)}({\bf x}) - H_{W'}^{(\ell)}({\bf x}) =\sum_{j=1}^\ell \Psi_{W_\ell}^{(\ell)}\circ \cdots\circ \Psi_{W_{j+1}}^{(j+1)}\circ\Psi_{W_{j}}^{(j)}\circ \left(\Psi_{W_{j-1}'}^{(j-1)}\circ\cdots \circ \Psi_{W_1'}^{(1)}({\bf x})\right) \\
- \Psi_{W_\ell}^{(\ell)}\circ \cdots\circ \Psi_{W_{j+1}}^{(j+1)}\circ \Psi_{W_j'}^{(j)}\left(\Psi_{W_{j-1}'}^{(j-1)}\circ\cdots \circ \Psi_{W_1'}^{(1)}({\bf x})\right).
\end{multline*}
Therefore,
\begin{multline*}
\left\|H_W^{(\ell)}({\bf x}) - H_{W'}^{(\ell)}({\bf x}) \right\|_2 \\
\leq \sum_{j=1}^\ell \prod_{k=j+1}^\ell \|W_k\|_{\textsf{op}}  \left\| \Psi_{W_j}^{(j)}\left(\Psi_{W_{j-1}'}^{(j-1)}\circ\cdots \circ \Psi_{W_1'}^{(1)}({\bf x})\right) - \Psi_{W_j'}^{(j)}\left(\Psi_{W_{j-1}'}^{(j-1)}\circ\cdots \circ \Psi_{W_1'}^{(1)}({\bf x})\right)\right\|_2\\
\leq  \sum_{j=1}^\ell \prod_{k=j+1}^\ell \|W_k\|_{\textsf{op}}  \|W_j - W_j'\|_{\textsf{op}} \left\|\Psi_{W_{j-1}'}^{(j-1)}\circ\cdots \circ \Psi_{W_1'}^{(1)}({\bf x})\right\|_2.
\end{multline*}
Since $\Psi^{(j)}_M({\bf 0}) = {\bf 0}$, we have the bound 
\[\left\|\Psi_{W_{j-1}'}^{(j-1)}\circ\cdots \circ \Psi_{W_1'}^{(1)}({\bf x})\right\|_2 \leq \|{\bf x}\|_2 \prod_{k=1}^{j-1} \|W_j'\|_{\textsf{op}}.\]
In conclusion, for all $1\leq \ell\leq D$, $W,W'\in\W$, and for all ${\bf x}\in\rset^{d_x}$, we have
\begin{equation}
\label{lip:h:2}
\|H^{(\ell)}_W( {\bf x}) - H^{(\ell)}_{W'}( {\bf x})\|_2 \leq \|{\bf x}\|_2 \sum_{j=1}^\ell \|W_j - W_j'\|_{\textsf{op}} \prod_{k=1}^{j-1}\|W_k'\|_{\textsf{op}} \prod_{k=j+1}^{\ell}\|W_k\|_{\textsf{op}}.\end{equation}

For ${\bf x}\in\rset^{d_x}$, ${\bf y}\in\Yset$, we define
\[F_{{\bf y}}({\bf x}) \eqdef \Prox^{\gamma\R}\left({\bf  x} - \gamma \nabla_{\bf x} f({\bf y}\vert {\bf x})\right), \;\;\mbox{ and }\;\; F_{{\bf y},W}({\bf x}) \eqdef H_W\left({\bf x} -\gamma \nabla_{\bf x} f({\bf y}\vert {\bf x})\right).\]
We use the notation $h^k$ to denote the function $h$ composed $k$ times  with the convention that $h^0$ is the identity map. Hence $g_W({\bf y}) = F_{{\bf y},W}^{D'}({\bf x}^{(0)})$. H\ref{H2:dl}  implies that $F_{\bf y}$ is non-expansive.
%, and by construction of the function class $H_W$, and the choice of $R_0$, for all $W\in\W$, and all ${\bf y}$, we have
%\begin{equation}\label{stab:F}
%\sup_{{\bf x}\in\rset^{d_x},\;\|{\bf x}\|_2\leq R_0} \|F_{{\bf y},W}({\bf x})\|_2 \leq R_0.\end{equation}
%In particular, since $\|{\bf x}^{(0)}\|_2\leq R_0$, (\ref{stab:F}) implies that $\|F_{{\bf y},W}^{j}({\bf x}^{(0)})\|_2 \leq R_0$ for all $j\geq 0$.
\begin{lemma}\label{lem:approx:gdn}
Assume H\ref{H2:dl}. Given $\epsilon>0$, we can find $W\in\W$ as described in H\ref{H2:dl} such that 
\begin{equation*}
\max_{1\leq i\leq n}\;\|g_W({\bf y}_i)  -  g({\bf y}_i) \|_2 \leq R_0\varrho_n^{D'}  + D'  \epsilon.\end{equation*}
\end{lemma}
\begin{proof}
For any ${\bf y}\in\rset^{d_y}$, we can write
\begin{equation*}
g({\bf y}) - g_W({\bf y}) = g({\bf y}) - F_{{\bf y},W}^{D'}({\bf x}^{(0)}) = g({\bf y})  - F_{{\bf y}}^{D'}({\bf x}^{(0)})+ F_{{\bf y}}^{D'}({\bf x}^{(0)})  - F_{{\bf y},W}^{D'}({\bf x}^{(0)}).\end{equation*}
By  Assumption \ref{H2:dl}, we have
\[
\max_{1\leq i\leq n}\left\|g({\bf y}_i) - F_{{\bf y}_i}^{D'}({\bf x}^{(0)})\right\|_2 \leq R_0\varrho_n^{D'}.\]
For ${\bf y}\in\rset^{d_y}$, let $G_{\gamma,{\bf y}}({\bf x}) \eqdef {\bf x} - \gamma \nabla_{\bf x} f({\bf y}\vert {\bf x})$, and
\[R_1' \eqdef \max_{1\leq i\leq n}\;\max_{{\bf x}\in\rset^{d_x}:\;\|{\bf x}\|_2 \leq R} \;\left\| G_{\gamma,{\bf y}_i}({\bf x})\right\|_2.\]
%Note that $G_{\gamma,{\bf y}}$ is non-expansive, and $\hat{{\bf x}}_i$ satisfies $G_{\gamma,{\bf y}_i}(\hat{{\bf x}}_i) = \hat{{\bf x}}_i$. If for some ${\bf x}\in\rset^{d_x}$, and some $R>0$, $\|{\bf x} - \hat{{\bf x}}_i\|_2\leq R$, then
%\[\left\|G_{\gamma,{\bf y}_i}({\bf x})  - \hat{{\bf x}}_i\right\|_2 = \left\|G_{\gamma,{\bf y}_i}({\bf x})  - G_{\gamma,{\bf y}_i}(\hat{{\bf x}}_i)\right\|_2\leq \|{\bf x}  - \hat{{\bf x}}_i\|_2 \leq R.\]
%Therefore,
%\begin{multline*}
%\left\|F_{{\bf y}_i,W}({\bf x}) - g({\bf y}_i)\right\|_2 = \\
%\left\| H_W\left(G_{\gamma,{\bf y}_i}({\bf x})\right) - \Prox^{\gamma \R}\left(G_{\gamma,{\bf y}_i}({\bf x})\right) +\Prox^{\gamma \R}\left(G_{\gamma,{\bf y}_i}({\bf x})\right) - \Prox^{\gamma \R}\left(G_{\gamma,{\bf y}_i}({g({\bf y}_i)})\right)\right\|_2\\
%\leq \left\| H_W\left(G_{\gamma,{\bf y}_i}({\bf x})\right) - \Prox^{\gamma \R}\left(G_{\gamma,{\bf y}_i}({\bf x})\right)\right\|_2 + \|{\bf x} - g({\bf y}_i)\|_2\\
%\leq \epsilon + \|{\bf x} - g({\bf y}_i)\|_2.
%\end{multline*}
%
%This implies that we can find a compact set $\Cset$ contained in a ball of radius $R_n$ (depending on the sample size) such that is ${\bf x}\in\Cset$, then $G_{\gamma,{\bf y}}({\bf x})\in\Cset$.
%
Since $F_{\bf y}$ is non-expansive, by the telescoping argument (\ref{telesc:id}), we have
\begin{multline*}
\left\|F_{{\bf y}_i}^{D'}({\bf x}^{(0)})  - F_{{\bf y}_i,W}^{D'}({\bf x}^{(0)})\right\|_2 \leq \sum_{j=1}^{D'}\;\|F_{{\bf y}_i,W}\left(F_{{\bf y}_i,W}^{j-1}({\bf x}^{(0)})\right) - F_{{\bf y}_i}\left(F_{{\bf y}_i,W}^{j-1}({\bf x}^{(0)})\right)\|_2,\\
\leq D' \sup_{{\bf x}\in\rset^{d_x},\;\|{\bf x}\|_2\leq R_1'}\; \|H_W({\bf x}) - \Prox^{\gamma\R}({\bf x})\|_2 \leq D'\epsilon.
\end{multline*}
The result follows by taking the max over $i$.
\end{proof}

\begin{lemma}\label{lem:lip:gdn}
Assume H\ref{H2:dl}, H\ref{H3:dl}, and let $\{g_W,\;W\in\W\}$ be as  in (\ref{fun:dnn}). For any $\eta>0$, and any $W,W'\in\W$, such that $\max(\|W\|_2,\|W'\|_2)\leq \eta$, we have
\[\max_{1\leq i\leq n}\;\|g_W({\bf y}_i) - g_{W'}({ \bf y}_i)\|_2 \leq L(\eta) \|W - W'\|_2,\]
with
\[L(\eta) \eqdef  C_n \left(e^4 + \frac{\eta^2}{D}\right)^{DD'},\]
and
\[C_n \eqdef \|{\bf x}^{(0)}\|_2 + \max_{1\leq i\leq n}\|{\bf x}^{(0)} - \gamma \nabla f({\bf y}_i\vert {\bf x}^{(0)})\|_2.\]
\end{lemma}
\begin{proof}
We recall that the convexity of ${\bf x}\mapsto f({\bf y}\vert{\bf x})$ and the choice of the step-size assumed in H\ref{H2:dl} imply that the function ${\bf x}\mapsto {\bf x} -\gamma \nabla_{\bf x} f({\bf y}\vert {\bf x})$ is non-expansive on $\rset^{d_x}$.  First, we apply (\ref{lip:h:1}) to obtain that for all ${\bf x}_1,{\bf x}_2\in\rset^{d_x}$,
\[\|H_W({\bf x}_1) - H_W({\bf x}_2)\|_2 \leq \lambda_W \|{\bf x}_1 - {\bf x}_2\|_2,\;\;\mbox{ where }\;\; \lambda_W\eqdef \prod_{\ell=1}^D \|W_\ell\|_{\textsf{op}}\vee 1.\]
It follows that for all ${\bf y}\in\Yset$, ${\bf x}_1,{\bf x}_2\in \rset^{d_x}$,
\begin{multline}\label{dif:F}
\|F_{{\bf y},W}({\bf x}_1) - F_{{\bf y},W}({\bf x}_2)\|_2\\
\leq \lambda_W \|{\bf x}_1 - {\bf x}_2 -\gamma\left(\nabla_{\bf x} f({\bf y}\vert {\bf x}_1) ) - \nabla_{\bf x} f({\bf y}\vert {\bf x}_2) \right)\|_2\leq \lambda_W\|{\bf x}_1 - {\bf x}_2\|_2.
\end{multline}
Using (\ref{dif:F}),  and the telescoping  identity (\ref{telesc:id}),  we obtain
\begin{equation}\label{lip:proof:eq3}
\|g_W({\bf y}) - g_{W'}({\bf y})\|_2 \leq \sum_{j=1}^{D'}\lambda_W^{D'-j} \left\|F_{{\bf y},W}\left(F_{{\bf y},W'}^{j-1}({\bf x}^{(0)})\right) - F_{{\bf y},W'}\left(F_{{\bf y},W'}^{j-1}({\bf x}^{(0)}\right)\right\|_2.
\end{equation}
We set $G_{\gamma,{\bf y}}({\bf x}) \eqdef {\bf x} - \gamma \nabla f({\bf y}\vert {\bf x})$. We apply  (\ref{lip:h:1}) with ${\bf x}_2=\mathbf{0}$ and the non-expansiveness of $G_{\gamma,{\bf y}}$ to write that for all ${\bf x}\in\rset^{d_x}$
\begin{multline*}
\|F_{{\bf y},W}({\bf x})\|_2 =  \left\|H_{W}\left(G_{\gamma,{\bf y}}({\bf x})\right)\right\|_2\leq \lambda_W \|G_{\gamma,{\bf y}}({\bf x})-G_{\gamma,{\bf y}}({\bf x}^{(0)}) +G_{\gamma,{\bf y}}({\bf x}^{(0)})\|_2 \\
\leq \lambda_W\left(\|{\bf x}\|_2 +\| {\bf x}^{(0)}\|_2 + \|G_{\gamma,{\bf y}}({\bf x}^{(0)}) \|_2\right).
\end{multline*}
By iterating this inequality we obtain that for all for all ${\bf x}\in\rset^{d_x}$
\begin{equation}
\|F^j_{{\bf y},W}({\bf x}^{(0)})\|_2 \leq \left(\sum_{\ell=1}^{j}\lambda_W^\ell\right)\left(\| {\bf x}^{(0)}\|_2 + \|G_{\gamma,{\bf y}}({\bf x}^{(0)}) \|_2\right) \leq C_n \sum_{\ell=1}^{j}\lambda_W^\ell,
\end{equation}
where
\[C_n \eqdef \|{\bf x}^{(0)}\|_2 + \max_{1\leq i\leq n}\|{\bf x}^{(0)} - \gamma \nabla f({\bf y}_i\vert {\bf x}^{(0)})\|_2.\]
Setting 
\[\lambda_{W,W'}\eqdef \sum_{j=1}^D \|W_j - W_j'\|_{\textsf{op}} \prod_{k=1}^{j-1}\|W_k'\|_{\textsf{op}} \prod_{k=j+1}^{D}\|W_k\|_{\textsf{op}},\]
  we then apply  (\ref{lip:h:2}) to write that for all ${\bf x}\in\rset^{d_x}$
\begin{multline*}
\left\|F_{{\bf y},W}\left(F_{{\bf y},W'}^{j-1}({\bf x}^{(0)})\right) - F_{{\bf y},W'}\left(F_{{\bf y},W'}^{j-1}({\bf x}^{(0)}\right)\right\|_2\\
= \left\|H_W\circ G_{\gamma,{\bf y}}\left(F_{{\bf y},W'}^{j-1}({\bf x}^{(0)})\right) - H_{W'}\circ G_{\gamma,{\bf y}}\left(F_{{\bf y},W'}^{j-1}({\bf x}^{(0)})\right)\right\|_2\\ 
\leq \lambda_{W,W'}\left\|G_{\gamma,{\bf y}}\left(F_{{\bf y},W'}^{j-1}({\bf x}^{(0)})\right)-G_{\gamma,{\bf y}}({\bf x}^{(0)})+G_{\gamma,{\bf y}}({\bf x}^{(0)})\right\|_2\\
\leq \lambda_{W,W'}\left(\|F_{{\bf y},W'}^{j-1}({\bf x}^{(0)})\|_2 + \|{\bf x}^{(0)}\|_2 + \|G_{\gamma,{\bf y}}({\bf x}^{(0)}\|_2\right)\\
\leq C_n\lambda_{W,W'}\sum_{\ell=0}^{j-1}\lambda_{W'}^\ell.
\end{multline*}
The last display together with (\ref{lip:proof:eq3}) yields,
\begin{equation}\label{lip:proof:eq4}
\max_{1\leq i\leq n}\|g_W({\bf y}_i) - g_{W'}({\bf y}_i)\|_2 \leq C_n\lambda_{W,W'}\sum_{j=1}^{D'} \sum_{\ell=0}^{j-1}\lambda_W^{D'-j}\lambda_{W'}^{\ell}.\end{equation}
Since the geometric mean is never larger than the arithmetic mean, we have
\[
\lambda_W \eqdef \prod_{j=1}^{D}1\vee\|W_j\|_{\textsf{op}}\leq \left(\frac{1}{D}\sum_{j=1}^D 1\vee\|W_j\|_{\textsf{op}}^2\right)^{D/2}\leq \left(1 + \frac{\|W\|_2^2}{D}\right)^{D/2}.\]
Therefore, for $\max(\|W\|_2,\|W'\|_2)\leq \eta$,
\[\sum_{j=1}^{D'} \sum_{\ell=0}^{j-1}\lambda_W^{D'-j}\lambda_{W'}^{\ell} \leq \sum_{j=1}^{D'}j\lambda_W^{D'-j}\lambda_{W'}^{j-1} \leq (D')^2 \left(1 + \frac{\eta^2}{D}\right)^{D(D'-1)/2},\]
and similarly,
\[\lambda_{W,W'} \leq \sqrt{D}\left(1 + \frac{2\eta^2}{D}\right)^{D/2}\|W-W'\|_2.\]
Hence, we conclude that
\begin{equation}\label{lip:proof:eq5}
\max_{1\leq i\leq n}\|g_W({\bf y}_i) - g_{W'}({\bf y}_i)\|_2 \leq C_n \sqrt{D}(D')^2\left(1 + \frac{2\eta^2}{D}\right)^{DD'/2}\|W-W'\|_2.
\end{equation}
The statement in the lemma follows by noting that 
\[\sqrt{D}(D')^2\left(1 + \frac{2\eta^2}{D}\right)^{DD'/2} \leq \sqrt{D}(D')^2\left(e^4 + \frac{2\eta^2}{D}\right)^{DD'/2}\leq \left(e^4 + \frac{2\eta^2}{D}\right)^{DD'},\]
using the fact that  $A^{x/2}\geq x$ for all $x\geq 1$, and $A\geq e$, and $A^{x/2}\geq x^2$ for all $x\geq 1$, and $A\geq e^4$.
\end{proof}

\subsubsection{Proof of Theorem \ref{thm:2}}

\begin{proof}
		We recall the notation $a\lesssim b$ means that $a\leq c b$, for some constant $c$ that does not depend on the sample size $n$. Fix 
		\[\varpi_\star = \log(n)\sqrt{d_x}\left(\frac{\log(q)}{n}\right)^{\frac{1}{2+\beta_2}},\;\;\;\mbox{ and }\;\;\;  \frac{\log\left(\frac{2R_0}{\varpi_\star}\right)}{-\log(\rho)} \leq D'\leq n.\]
		By Assumption \ref{H3:dl}, and Lemma \ref{lem:approx:gdn}, by taking a deep neural network function $H_W$, with depth $D=D_0\log(2D'\sqrt{d_x}/\varpi_\star)$, layer size $(p_0,\ldots,p_D)$ all at most $N_0(2D'\sqrt{d_x}/\varpi_\star)^{\beta_1}$, and $W\in\W$ with sparsity at most $s_\star = s_0(2D'\sqrt{d_x}/\varpi_\star)^{\beta_2}$, and
		we achieve
		\begin{multline*}
			\max_{1\leq i\leq n}\;\|g_W({\bf y}_i)  -  g({\bf y}_i) \|_2 \leq R_0\rho^{D'} \\
			+ D' \;\sup_{{\bf x}:\;\|{\bf x}\|_2\leq R_0'}\; \|H_W({\bf x}) - \Prox^{\gamma\R}({\bf x})\|_2\\
			\leq  R_0\frac{\varpi_\star}{2R_0} + D'\frac{\varpi_\star}{2D'} = \varpi_\star.\end{multline*}
		Then by Lemma \ref{lem:lip:gdn}, the term $\textsf{L}_\star$ in Theorem \ref{thm:main} scales like		
		\begin{align*}
			\textsf{L}_\star \simeq &\left(e^4 + \frac{ s_\star}{D}\right)^{DD'} s_\star^{1/2}\\
			\lesssim &\left(e^4 + \frac{ s_\star}{D}\right)^{DD'} \left(1 + \frac{ s_\star}{D}\right)^{D/2}\\
			\simeq& \left(e^4 + \frac{ s_\star}{D}\right)^{D(D'+1/2)}\\
			\lesssim& \left(e^4 + q\right)^{D(D'+1/2)}
		\end{align*}

		It follows that
		\[\frac{\log\mathsf{L}_\star}{\log(q)} \lesssim DD'\lesssim D'\log(n),\;\;\mbox{ and }\;\; s_\star = s_0\left(\frac{2D'\sqrt{d_x}}{\varpi_\star}\right)^{\beta_2} \lesssim   \left(\frac{D'}{\log(n)}\right)^{\beta_2}\left(\frac{n}{\log(q)}\right)^{\frac{\beta_2}{2+\beta_2}},\]
		and the term $s$ in Theorem \ref{thm:main} is of order
		\begin{align*}
			s \lesssim &  s_\star\left(\frac{\log(n)}{\log(q)} + \frac{\log(\mathsf{L}_\star)}{\log(q)}\right)  + \frac{n\varpi_\star^2}{\log(q)} \\
			\lesssim & \left[\left(1+ D'\log(n)\right)\left(\frac{D'}{\log(n)}\right)^{\beta_2} + (\log(n))^2\right] \left(\frac{n}{\log(q)}\right)^{\frac{\beta_2}{2+\beta_2}}\\
			\lesssim &\left(D'\right)^{1+\beta_2} (\log(n))^2\left(\frac{n}{\log(q)}\right)^{\frac{\beta_2}{2+\beta_2}}.
		\end{align*}
		Therefore the term $\textsf{L}_s$ in Theorem \ref{thm:main} scales like
		\[\mathsf{L}_s = L(s^{1/2}b_s) \lesssim\left(e^4 + \frac{ sb_s^2}{D}\right)^{D(D'+1/2)},\]
		which gives,
		\[%\log(s^{1/2}b_s\mathsf{L}_s) \lesssim 
		\log(\mathsf{L}_s) \lesssim DD' \log(q) \lesssim D'\log(n)\log(q).\]
		Noting that
		\[\frac{s}{n} \lesssim \left(D'\right)^{1+\beta_2} \left(\frac{\log(q)}{n}\right)^{\frac{2}{2+\beta_2}} \frac{(\log(n))^2}{\log(q)},\]
		we deduce that the conclusion of Theorem \ref{thm:main} holds with a rate
		\begin{equation*}
			\r = \bar\sigma\sqrt{\frac{s\log(q) + s\log( \mathsf{L}_s)}{n}}
			\lesssim  \bar\sigma\left(D'\right)^{1+\beta_2/2}  (\log(n))^{3/2}\left(\frac{\log(q)}{n}\right)^{\frac{1}{2+\beta_2}}.\end{equation*}
		
	\end{proof}

\end{document}